\documentclass[letterpaper]{article} 
\usepackage{aaai2026}  

\usepackage{times}  
\usepackage{helvet}  
\usepackage{courier}  
\usepackage[hyphens]{url}  
\usepackage{graphicx} 
\usepackage{natbib}  
\usepackage{caption} 
\usepackage{subcaption} 

\usepackage{algorithm}
\usepackage{algorithmic}

\usepackage{newfloat}
\usepackage{listings}
\DeclareCaptionStyle{ruled}{labelfont=normalfont,labelsep=colon,strut=off} 
\floatstyle{ruled}
\newfloat{listing}{tb}{lst}{}
\floatname{listing}{Listing}
\usepackage{latexsym, amsfonts, amssymb, amsmath, amsthm, mathrsfs, mathtools}

\numberwithin{equation}{section}
\newtheorem{thm}{Theorem}[section]
\newtheorem{lem}[thm]{Lemma}
\newtheorem{prop}[thm]{Proposition}
\newtheorem{cor}[thm]{Corollary}

\theoremstyle{remark}
\newtheorem{rem}[thm]{Remark}

\newtheorem{assumption}{Assumption}

\newcommand{\vertiii}[1]{{\left\vert\kern-0.25ex\left\vert\kern-0.25ex\left\vert #1 
    \right\vert\kern-0.25ex\right\vert\kern-0.25ex\right\vert}}

\newcommand{\RR}{\mathbb{R}}

\newcommand{\RD}{\mathbb{R}^{d}}
\newcommand{\EE}{\mathbb{E}\,}

\newcommand{\calF}{\mathcal{F}}

\newcommand{\calA}{\mathcal{A}}
\newcommand{\calR}{\mathcal{R}}
\newcommand{\calQ}{\mathcal{Q}}
\newcommand{\calG}{\mathcal{G}}

\newcommand{\nnset}{\mathcal{F}(L, \mathbf{p}, s, F)}
\newcommand{\Iamn}{\mathbf{1}_{\mathcal{A}_{t_m}^{(n)}}}

\newcommand{\dd}{\mathrm{d}}

\title{Drift Estimation for Diffusion Processes Using Neural Networks Based on Discretely Observed Independent Paths}
\author{
    Yuzhen Zhao\textsuperscript{\rm 1},
    Yating Liu\textsuperscript{\rm 2},
    Marc Hoffmann\textsuperscript{\rm 3}
}
\affiliations {
    \textsuperscript{\rm 1}Sorbonne Université, and LEDa, Université Paris-Dauphine, PSL, Paris, France\\
    \textsuperscript{\rm 2}CEREMADE, CNRS, Universit\'e Paris-Dauphine, PSL,  75016 Paris, France\\
    \textsuperscript{\rm 3}CEREMADE, CNRS, Universit\'e Paris-Dauphine, PSL and Institut Universitaire de France,  75016 Paris, France\\
    yuzhen.zhao@dauphine.eu,\: liu@ceremade.dauphine.fr, \: hoffmann@ceremade.dauphine.fr
}

\begin{document}

\maketitle
\begin{abstract}
This paper addresses the nonparametric estimation of the drift function over a compact domain for a time-homogeneous diffusion process, based on high-frequency discrete observations from $N$ independent trajectories. We propose a neural network-based estimator and derive a non-asymptotic convergence rate, decomposed into a training error, an approximation error, and a diffusion-related term scaling as ${\log N}/{N}$. For compositional drift functions, we establish an explicit rate. In the numerical experiments, we consider a drift function with local fluctuations generated by a double-layer compositional structure featuring local oscillations, and show that the empirical convergence rate becomes independent of the input dimension $d$. Compared to the $B$-spline method, the neural network estimator achieves better convergence rates and more effectively captures local features, particularly in higher-dimensional settings.
\end{abstract}

\begin{links}
    \link{Code}{https://github.com/yuzhen3001/nn-drift-estimation-diffusion-process} 
     \link{Extended version}{https://arxiv.org/abs/2511.11161}
\end{links}

\section{Introduction}

In this paper, we study the estimation of the drift function of a time-homogeneous diffusion process using neural networks. Specifically, we consider an $\mathbb{R}^d$-valued diffusion process $X=(X_t)_{t\in[0,T]}$ that solves the stochastic differential equation (SDE)
\begin{equation}\label{eq:sde}
\dd X_t = b(X_t)\,\dd t + \sigma(X_t)\,\dd B_t,
\end{equation}
where $X_0$ is a random vector, $B = (B_t)_{t \in [0, T]}$ is a $d$-dimensional standard  Brownian motion, and  $b : \mathbb{R}^d \to \mathbb{R}^d$ and $\sigma : \mathbb{R}^d \to \mathbb{R}^{d}\otimes \RD$ are measurable functions. To estimate the drift, we assume that $N$ independent and identically distributed (i.i.d.) sample paths $\bar{X}^{(n)} = (\bar X^{(n)}_t)_{t \in [0, T]}$, $1 \le n \le N$, are available, each having the same distribution as $X$ and independent of $X$. Moreover, each sample path is observed at high frequency, that is, at discrete time points on a fine grid:
\begin{equation}\label{eq:train-set}
   \bar{X}^{(n)}_{t_0 : t_M} \coloneqq \big( \bar X^{(n)}_{t_0=0}, \bar X^{(n)}_{t_1}, \dots, \bar X^{(n)}_{t_M=T} \big), \;1\!\leq\! n\!\leq \!N, 
\end{equation}
where $t_m\!=\!m \Delta$ with the time step $\Delta \!=\! \frac{T}{M} \!\to\! 0$ as $M\!\to\!\infty$.
\subsection{Literature Review}

Diffusion processes, as defined by \eqref{eq:sde}, are fundamental stochastic models with various applications in physics, biology, and mathematical finance (see, e.g., \citet{Gardiner2004, Bressloff2014, Karatzas1998}). A central task is the estimation of the drift function, which has been extensively studied under two standard frameworks: the long-time horizon regime, often assuming ergodicity (see, e.g.,  \citet{Hoffmann1999, Kutoyants2004, Dalalyan2005, Frishman2020}), and the high-frequency regime over a fixed time interval (see, e.g., \citet{Comte2020, Denis2021,8516991}).

In recent years, neural networks have become widely used as effective tools for nonparametric function estimation in statistical learning (see, e.g., \citet{Yarotsky2017, SchmidtHieber2020}). They have also been successfully applied to drift estimation in diffusion processes under the long-time horizon setting with ergodicity assumptions (see, e.g., \citet{Oga2024, ren2024statistical}). On the other hand, the neural network-based drift estimation in the high-frequency regime over a fixed time interval is beginning to be explored (see, e.g., \cite{gao2024learning, BAE2025116440}), and theoretical analyses remain limited. Our work contributes to this direction by providing convergence guarantees. 


\subsection{Contribution and Organization of This Paper}

In this work, we address the problem of estimating the drift function $b$ over a compact set $K \subset \mathbb{R}^d$, using neural networks defined further in \eqref{eq:sparseNN}, based on discrete observations of the process over a fixed and finite time horizon $T$. Importantly, our approach does not rely on the ergodicity of the diffusion process, in contrast to \citet{Oga2024}. Moreover, since the drift function can be estimated component-wise, we can focus on estimating each component separately and combine these estimators to obtain an estimator for $b$. Specifically, for a given $i \in \{1, \dots, d\}$, we consider the $i$-th component of $b$ restricted to the compact domain $[0,1]^d$.  Our estimation target is therefore the function 
\begin{equation}\label{eq:estimation-target}
f_0 \coloneqq b^i \mathbf{1}_{[0,1]^d},
\end{equation} 
where $\mathbf{1}_{[0,1]^d}$ denotes the indicator function of the domain $[0,1]^d$.
It is worth noting that the assumption of a compact domain is standard in the literature (see, e.g., \citet{Hoffmann1999, Comte2007, Oga2024}). 
In this paper, the choice of $[0,1]^d$ is made for convenience in theoretical analysis and can be replaced by any compact set $K \subset \mathbb{R}^d$. In practice, the compact set $K$ can be determined from the training data distribution (e.g., based on sample coverage or through a sample-splitting procedure).

The main contribution of this paper is Theorem~\ref{thm:main}, which establishes a non-asymptotic upper bound on the $L^2$ estimation risk in terms of the training error from neural network optimization, the approximation error over the neural network class, the number of training samples $N$, and the time step $\Delta$ of the observation grid.  Moreover, under a composition structure assumption on $f_0$, as considered in \citet{SchmidtHieber2020} and \citet{Oga2024}, we derive an explicit upper bound on the test error as a function of $N$.

The paper is organized as follows. In Section~\ref{sec:math-setting}, we present the mathematical framework, along with the assumptions and main result (Theorem~\ref{thm:main}). Section~\ref{sec:numerical-part} provides a simulation study illustrating the proposed method, where we compare the performance of our estimator with that of \citet{Denis2021}, which uses a $B$-spline-based estimator and ridge estimation. The numerical results show that our estimator outperforms the reference method in \citet{Denis2021} in terms of convergence rate and the ability to capture the local fluctuations, especially in high-dimensional settings. Finally, a sketch of the proof of the main result is provided in Section~\ref{sec:proof-sketch}. 

\subsection{Notations}

For a vector or matrix $W$, let $|W|$ denote the Euclidean norm (if $W$ is a vector) or the Frobenius norm (if $W$ is a matrix). We write $|W|_{\infty}$ for the maximum-entry norm and $|W|_0$ for the number of nonzero entries. For $\beta \in \RR$, $\lfloor \beta \rfloor$ denotes the largest integer \textit{strictly} smaller than $\beta$. For two sequences $(a_N)$ and $(b_N)$, we write $a_N\lesssim b_N$ if there exists a constant $C$ such that $a_N\leq C b_N$ for all $N$, and we write $a_N\asymp b_N$ if $a_N\lesssim b_N$ and $b_N\lesssim a_N$. 
Moreover, for a function $f : \mathbb{R}^d \to \mathbb{R}$, let $\Vert f\Vert_{\infty}\coloneqq \sup_{x\in[0,1]^d}|f(x)|$. 
For a random variable $X$, we write $\mathcal{L}(X)$ for its law. The index $m = 0, \dots, M-1$ corresponds to time steps on the discrete observation grid, while $n = 1, \dots, N$ indexes the samples in the training set.

\section{Mathematical Setting and Main Result}\label{sec:math-setting}

This section presents the mathematical setting and main result of the paper. We begin by introducing the neural network class used for estimation, followed by a description of the proposed estimator and the different loss functions considered.  Finally, we present the assumptions and main result of the paper.


\subsection{Neural Network Class for Estimation}

We use feedforward neural networks to estimate the drift function component-wise. Specifically, let $\mathcal{F}_{L, \mathbf{p}}$ denote the class of functions defined by neural networks with $L$ hidden layers, layer widths given by $\mathbf{p} = (p_0, p_1, \dots, p_{L+1})\in\mathbb{N}^{L+2}$, where $p_0 = d$ is the input dimension and  $p_{L+1} = 1$ is the output dimension. Each $f \in \mathcal{F}_{L, \mathbf{p}}$ is a function from $\RD$  to $\RR$ of the form
\begin{equation}
\label{eq:NN}
f(x)=W_{L} \sigma_{\mathbf{v}_{L}} W_{L-1} \sigma_{\mathbf{v}_{L-1}} \cdots W_{1} \sigma_{\mathbf{v}_{1}} W_{0} x,
\end{equation}
where  $W_{i}$  is a  $p_{i+1} \times p_{i}$  weight matrix, $\sigma(x)=\max(x, 0)$ is the ReLU function applied component-wise, and for $\mathbf{v}=(v_1, ..., v_r)\in\RR^r$, $\sigma_{\mathbf{v}} : \RR^r\rightarrow\RR^r$ denotes the shifted ReLU function with shift vector $\mathbf{v}$, defined by  
\begin{equation}
\sigma_{\mathbf{v}}\big(\,(y_1, \dots, y_r)^{\top}\,\big)=\big(\sigma\left(y_{1}-v_{1}\right), \dots, 
\sigma(y_{r}-v_{r})\big)^{\top}.\nonumber
\end{equation}


In this paper, we consider neural networks with sparsity parameter $s$, and we restrict the neural network functions to be uniformly bounded by a constant $F>0$. Since the target function $f_0$ is supported on $[0,1]^d$ (see \eqref{eq:estimation-target}), we define the class of neural network estimators $\mathcal{F}(L, \mathbf{p}, s, F)$ as follows:
\begin{align}
&\mathcal{F}(L, \mathbf{p}, s, F) 
:= \Big\{ f\mathbf{1}_{[0,1]^d} \::\: f\in \mathcal{F}_{L, \mathbf{p}} \text{ such that }\nonumber\\
&\qquad   \max_{j=0,\dots,L} \big( \|W_j\|_{\infty} \vee |\mathbf{v}_j|_{\infty} \big) \le 1,\; \big\| f \big\|_{\infty} \le F,\nonumber \\
&\qquad   \text{ and } \sum_{j=0}^{L} \big( \|W_j\|_0 + |\mathbf{v}_j|_0 \big) \le s \Big\}.
\label{eq:sparseNN}
\end{align}
When there is no ambiguity, we abbreviate $\mathcal{F}(L, \mathbf{p}, s, F)$ as $\calF$ for simplicity, and we also use the notation $\mathcal{F}(L,\mathbf{p},s)\coloneqq\mathcal{F}(L,\mathbf{p},s, \infty)$.



\subsection{Estimator and Loss Functions}

Recall that our training data $\mathcal{D}_N$ for this estimation problem consists of discrete observations from independent sample paths of solutions to \eqref{eq:sde}, that is,
\begin{equation}
\mathcal{D}_N=\Big\{\bar{X}_{t_0:t_M}^{(n)}=\big( \bar X^{(n)}_{t_0},\dots, \bar X^{(n)}_{t_M} \big), \:1\!\leq \!n \!\leq\! N\Big\}.\nonumber
\end{equation}
For a fixed $n$, let $\bar{X}^{(n),i}$ denote the $i$-th component of $\bar{X}^{(n)}$. The classical approach for estimating $b^i$ considers the increment
\begin{equation}\label{eq:defY}
    Y_{t_m}^{(n)} \coloneqq \frac{1}{\Delta}\left(\bar{X}^{(n),i}_{t_{m+1}}-\bar{X}^{(n),i}_{t_{m}}\right),
\end{equation}
which is approximately equal to $b^i(\bar{X}^{(n)}_{t_{m}})$ plus a noise term (see \citet[Appendix C]{Denis2021}, \citet[Equation (4.1)]{Oga2024}; see also \eqref{eq:decompos-Y}). Unlike in classical regression with neural networks (see \citet{SchmidtHieber2020}), the errors in our setting are neither normally distributed nor independent of $b^i(\bar{X}^{(n)}_{t_{m}})$, which presents a key challenge in analyzing the estimator’s performance.

Based on the above definition of $Y_{t_m}^{(n)}$, our neural network training procedure consists in minimizing
\begin{equation}\label{eq:train-nn-loss}
    \calQ_{\mathcal{D}_N}(f)\coloneqq \frac{1}{NM}\sum_{n=1}^{N} \sum_{m=0}^{M-1} \left(Y^{(n)}_{t_m} - f(\bar{X}^{(n)}_{t_m}) \right)^2
\end{equation}
over $f\in\mathcal{F}(L, \mathbf{p}, s, F)$. In practice, finding an exact minimizer of $\calQ_{\mathcal{D}_N}(f)$ may be challenging. To quantify the discrepancy between our estimator $\hat{f}=\hat{f}_{N}$, which is the output of the neural network training procedure applied to the empirical loss $\calQ_{\mathcal{D}_N}(f)$, and an exact minimizer of $\calQ_{\mathcal{D}_N}(f)$ over $\mathcal{F}(L, \mathbf{p}, s, F)$, we define
\begin{equation}\label{eq:def-psi}
    \Psi^{\mathcal{F}}\big(\hat{f}\,\big)\coloneqq \EE\left[\calQ_{\mathcal{D}_N}(\hat{f}\,)-\!\!\inf_{f\in\mathcal{F}(L, \mathbf{p}, s, F)}\calQ_{\mathcal{D}_N}(f)\right].
\end{equation}

The performance of $\hat{f}$ is measured by the estimation risk
\begin{equation}\label{eq:test-error}
\mathcal{R}(\hat{f}, f_0)\coloneqq \EE \!\left[\frac{1}{M}\sum_{m=0}^{M-1}\Big(\widehat{f}(X_{t_m})-f_0(X_{t_m})\Big)^2\right]. 
\end{equation}
where $X=(X_t)_{t\in[0,T]}$ is a solution to \eqref{eq:sde} that is independent of the training set $\mathcal{D}_{N}$. In practice (see also the numerical experiments in Section~\ref{sec:numerical-part}), the estimation risk $\mathcal{R}(\hat{f}, f_0)$ is often estimated by the following empirical test error
\begin{equation}\label{eq:test-empirical-error}
\widetilde{\mathcal{R}}(\hat{f}, f_0)\coloneqq \frac{1}{MN'}\!\sum_{n=1}^{N'}\sum_{m=0}^{M-1}\!\Big(\widehat{f}(\widetilde{X}_{t_m}^{(n)})-f_0(\widetilde{X}_{t_m}^{(n)})\Big)^2\!,
\end{equation}
where $\widetilde{X}^{(n)}=(\widetilde{X}^{(n)}_t)_{t\in[0,T]}, \,1\leq n\leq N',$ are i.i.d. sample paths, independent of the training set $\mathcal{D}_{N}$.


\subsection{Assumptions and Main Result}
Throughout this paper, we impose the following two assumptions. Assumption~\ref{assum:lip} ensures the existence and uniqueness of the solution to the SDE~\eqref{eq:sde}, while Assumption~\ref{assum:parameter} is a technical condition that guarantees the convergence result stated in Theorem~\ref{thm:main}.


\begin{assumption}
\label{assum:lip} The initial random variable $X_0$ satisfies $\EE[|X_0|^2]<+\infty$. 
The coefficient functions $b$ and $\sigma$ are globally Lipschitz continuous; that is, there exist constants $L_b, L_{\sigma} > 0$ such that for all $x, y \in \mathbb{R}^d$,
\begin{align*}
|b(x)-b(y)| \leq L_b |x-y|, \quad |\sigma(x)-\sigma(y)| \leq L_{\sigma} |x-y|.
\end{align*}
\end{assumption}

Since $b$ and $\sigma$ are continuous, we define
\[C_b := \sup_{x \in [0,1]^d} |b(x)| < \infty\text{ and }C_{\sigma}:=\sup_{x \in [0,1]^d} |\sigma(x)| < \infty.\]
It is obvious that $\Vert f_0\Vert_{\infty}\leq C_b$. 
\begin{assumption}
\label{assum:parameter}
$F\geq \max(C_b,1)$, $C_\sigma>0$, $L\geq1$, $s\geq2$, $N\geq2$ and $\Delta \leq 1$. 
\end{assumption}




Our main result is presented in the following theorem.

\begin{thm}\label{thm:main}
Suppose that Assumptions \ref{assum:lip} and \ref{assum:parameter} hold.
There exists a constant $\mathfrak{C}$, depending only on $C_b$, $C_\sigma$, $L_b$, $L_\sigma$, $T$, and a universal constant $C$ (introduced later in Lemma~\ref{lem:lemma413-in-oga}), such that
\begin{align}
&\mathcal{R}(\hat{f},f_{0})
\leq 4\Psi^{\calF}(\hat{f})+ 6\inf_{f\in\nnset}\calR(f,f_0)\nonumber\\
&+\mathfrak{C}F^2\left( \Delta +\frac{s(L \log s + \log d)+s\log 4F}{N}+s\frac{\log N}{N}\right). \nonumber
    \end{align}
\end{thm}

\begin{rem}
Theorem~\ref{thm:main} provides a decomposition of the upper bound on the estimation risk into three components. The first term reflects the error arising from the training procedure for finding a global minimizer with respect to the empirical loss $\calQ_{\mathcal{D}_N}$. The second term can be bounded by the classical regression error for i.i.d. samples, as analyzed in \citet{SchmidtHieber2020}, and does not depend on the diffusion setting. Finally, the third term captures the error due to the temporal dependence inherent in the diffusion process.
\end{rem}

Now let 
\begin{align}
&\mathcal{C}_{r}^{\beta}(D, K)=\bigg\{f:D\subset \RR^r\rightarrow \RR :
\sum_{\boldsymbol{\alpha}:|\boldsymbol{\alpha}|<\beta}\Vert \partial^{\boldsymbol{\alpha}}f\Vert_{\infty}\nonumber\\
&\qquad +\sum_{\boldsymbol{\alpha}:|\boldsymbol{\alpha}|=\lfloor\beta\rfloor}\sup_{x,y\in D, x\neq y}\frac{|\partial^{\boldsymbol{\alpha}}f(x)-\partial^{\boldsymbol{\alpha}}f(y)|}{|x-y|_{\infty}^{\beta - \lfloor\beta\rfloor}}\leq K\bigg\}.\nonumber
\end{align}
Let $\mathcal{G}(q,\mathbf{d}, \mathbf{t},\boldsymbol{\beta}, K)$ be the function space defined in \citet{SchmidtHieber2020}:
\begin{align}
&\mathcal{G}(q,\mathbf{d}, \mathbf{t},\boldsymbol{\beta}, K)\coloneqq \big\{ f=g_q\circ \dots g_0:\nonumber\\
&\qquad  g_i=(g_{ij})_j:[a_i, b_i]^{d_i}\rightarrow [a_{i+1}, b_{i+1}]^{d_{i+1}}, \nonumber\\
&\qquad  g_{ij}\in\mathcal{C}_{t_i}^{\beta_i}([a_i, b_i]^{t_i}, K), \text{ for some }|a_i|, |b_i|\leq K\big\}.\nonumber
\end{align}
with $\mathbf{d}\coloneqq (d_0, \dots, d_{q+1})$, $\mathbf{t}\coloneqq(t_0, \dots, t_q)$, $\boldsymbol{\beta}\coloneqq (\beta_0,\dots, \beta_q)$.
Define $\beta_i^*\coloneqq \beta_i\prod_{l=i+1}^{q}(\beta_l\wedge 1)$ and 
\begin{equation}\label{eq:def-phi-N}
    \phi_{N}\coloneqq \max_{0\leq i \leq q}N^{-\frac{2\beta_i^*}{2\beta_{i}^*+t_i}}.
\end{equation}

\begin{cor}\label{cor:cor-of-thm-main} Suppose that Assumptions \ref{assum:lip} and \ref{assum:parameter} hold.
Assume moreover that $f_0\in \mathcal{G}(q,\mathbf{d}, \mathbf{t},\boldsymbol{\beta}, K)$ and the neural network estimators set $\nnset$ satisfies 
\begin{enumerate}
    \item[$\mathrm{(i)}$] $F\geq \max(K,1)$, $L\asymp \log_2N$ 
    \item[$\mathrm{(ii)}$] $N\phi_N\lesssim \min_{i=1, ..., L}p_i$, $s\asymp N\phi_N \log N$.
\end{enumerate}
Then there exists a constant $C$  depending  on $q, \mathbf{d}, \mathbf{t}, \boldsymbol{\beta}, F, C_b, C_\sigma, L_b, L_\sigma, T$ such that if \[\Delta \lesssim \phi_N\log^3N \text{ and }\Psi^{\calF}(\hat f)\leq C \phi_N\log^3N,\] then $\calR (\hat f, f_0)\leq C \phi_N \log^3 N$.
\end{cor}

\section{Numerical Experiments}\label{sec:numerical-part}

In this section, we evaluate the performance of our estimator using the following drift functions:
\begin{align}\label{eq:b-example}
x\in\RD \mapsto b(x) = -x+\phi\left(\tfrac{s(x)}{\theta}\right)\mathbf{1}_d,  
\end{align}
where $\mathbf{1}_d=(1, \dots, 1)^{\top}$, $s(x)\coloneqq\sum_{i=1}^{d} x_i$, $\theta=0.2$ and $\phi(z) \coloneqq 2z\exp(-z^2)-8z\exp(-2z^2)$. In the following, Figure~\ref{fig:capture-fluctuations}a \textit{(True function)} illustrates 
$b(x)$ in one dimension. The initial condition $X_0$ is set as a random variable having standard normal distribution $\mathcal{N}(0, I_{d})$, and the diffusion coefficient $\sigma(x)$ is set to the identity matrix. 

The numerical study consists of three parts. First, we examine the convergence rate of our estimator with respect to $N$. Then, we compare its performance to the $B$-spline-based estimator proposed in \citet{Denis2021}, focusing on two aspects: the convergence rate and the ability to capture the local fluctuations induced by the function $\phi$ in the drift. Finally, we conclude this section with a summary of the comparative results, followed by additional remarks on the memory requirements and computational feasibility of our estimator compared to the $B$-spline-based approach.


\subsection{Experimental Setup}

Throughout this section, we fix the time horizon to $T = 1$ and the time step to $\Delta = 0.01$, and we estimate the first component of the drift coefficient $b_1(x) : \mathbb{R}^d \rightarrow \mathbb{R}$. We set the test set size $N'=1000$, and the estimation risk is approximated using its empirical version defined in~\eqref{eq:test-empirical-error}. The experiments are carried out for a range of dimensions $d \in \{1, 2, 10, 50\}$ and training size $N \in \{100, 200, 500, 1000, 2000, 5000\}$. 

For the neural network estimator, we perform hyperparameter tuning for each combination of $ (N, d) $. We consider different hidden layer width 
\[\mathbf{p}\in \{(d, 16, 16, 1), (d, 16, 32, 16, 1), (d, 16, 32, 32, 16, 1)\}.\] The number of non-zero parameters $s$ is set as a proportion $s_{\text{ratio}}$ of the total number of parameters, with $s_{\text{ratio}} \in \{0.25, 0.5, 0.75\}$. For the training procedure, we use the Adam optimizer with a learning rate of $ 10^{-3} $.  The final number of training epochs is determined via early stopping triggered after 20 consecutive epochs.

Moreover, for each configuration, we repeat the experiment 50 times independently and compute the average error along with the corresponding 95\% confidence interval. All implementations are performed using  PyTorch.

\subsection{Numerical Study of the Convergence Rate}

We first examine how the estimation error evolves with the number of training trajectories $N$. In the setting of Corollary \ref{cor:cor-of-thm-main}, we have $\phi_N=N^{-1}$ for the first component of the drift function $b$ defined in \eqref{eq:b-example}. Hence, if 
$\Psi^{\calF}(\hat f)\leq C \phi_N\log^3N$,  then the theoretical upper bound of $\calR (\hat f, f_0)$ will be $CN^{-1}\log^3 N$ for some constant $C$.  

Figure~\ref{fig:convergence} shows the log–log plot of the estimation risk $\calR (\hat f, f_0)$ as a function of $N$ for $d \in \{1,2 ,10, 50\}$. The empirical convergence rate closely matches the theoretical upper bound for $d\in\{1,2,50\}$. For $d=10$, however, the empirical error slightly exceeds the theoretical envelope for certain values of $N$. This deviation likely results from the restricted hyperparameter search space, which covered only a limited range of learning rates, training epochs, network widths and depths, and regularization parameters. However, the search grid was intentionally constrained to ensure comparability and to maintain a reasonable computational cost across dimensions.

\subsection{Comparison with $B$-Spline-Base Estimator}

We now compare our estimator with the $B$-spline-based estimator from \citet{Denis2021}, using the same training and test datasets. The implementation of the $B$-spline estimator follows the procedure described in \citet{Denis2021}, with an adaption to the multivariate setting using tensor-product $B$-splines basis (see e.g. \citet{gyorfi2002distribution}), and the number of knots selected based on validation performance.

Figure~\ref{fig:compare-convergence} demonstrates that our method achieves a better convergence rate of the estimation risk in $N$. 
Furthermore, as illustrated in Figure~\ref{fig:capture-fluctuations}, our estimator ($\mathbf{p} =  (d, 16, 32, 16, 1), s_{\text{ratio}}
 = 0.75$) more accurately captures the local variations of the drift function, while the $B$-spline-based method fails to recover these fine-scale features.

\begin{figure}[H]
    \centering
    \begin{subfigure}{0.49\linewidth}
        \includegraphics[width=\linewidth]{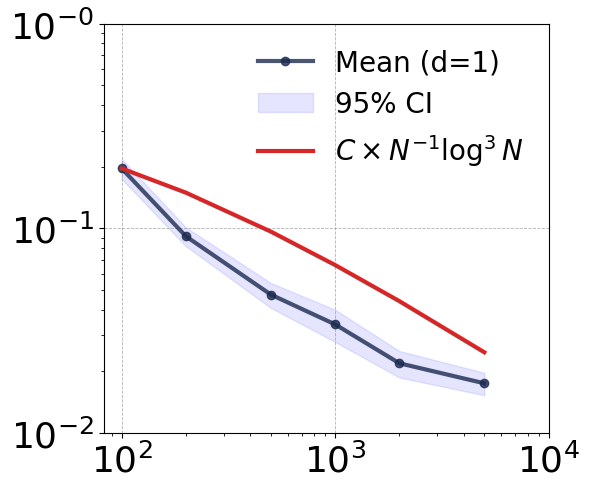}
        \label{fig:sub1}
    \end{subfigure}
    \hfill
    \begin{subfigure}{0.49\linewidth}
        \includegraphics[width=\linewidth]{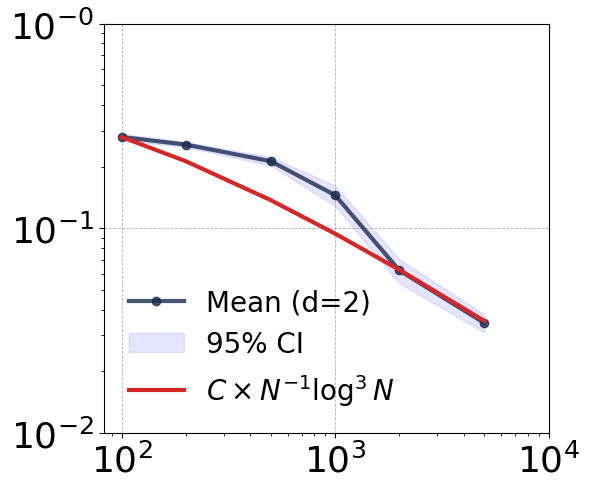}
        \label{fig:sub2}
    \end{subfigure}

    \begin{subfigure}{0.49\linewidth}
        \includegraphics[width=\linewidth]{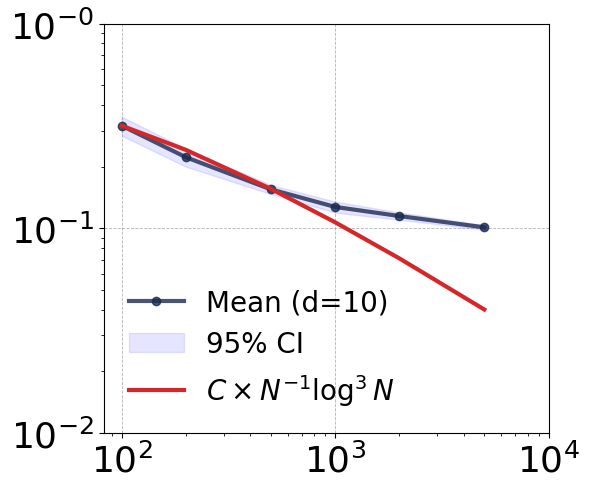}
        \label{fig:sub4}
    \end{subfigure}
    \hfill
    \begin{subfigure}{0.49\linewidth} 
        \includegraphics[width=\linewidth]{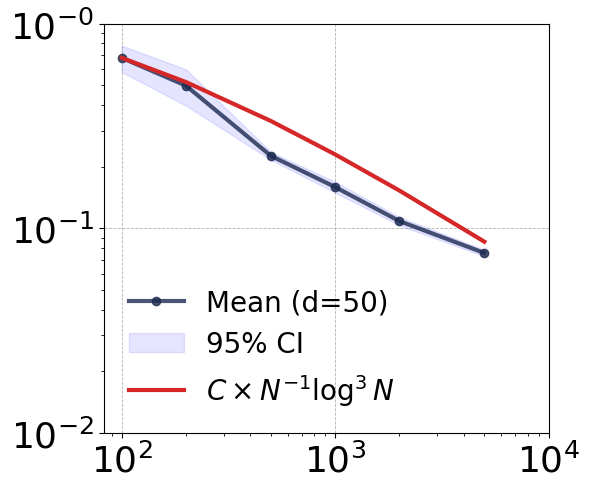}
        \label{fig:sub5}
    \end{subfigure}

    \caption{Convergence rates of our estimator for $ d = 1 $ (top left), $ d = 2 $ (top right), $ d = 10 $ (bottom left), and $ d = 50 $ (bottom right). 
}
    \label{fig:convergence}
\end{figure}


\begin{figure}[H]
    \centering
    \begin{subfigure}{0.49\linewidth}
        \includegraphics[width=\linewidth]{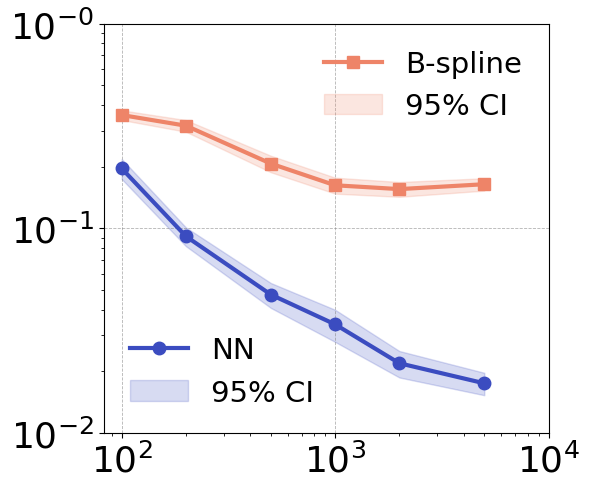}
        \label{fig:sub1}
    \end{subfigure}
    \hfill
    \begin{subfigure}{0.49\linewidth}
        \includegraphics[width=\linewidth]{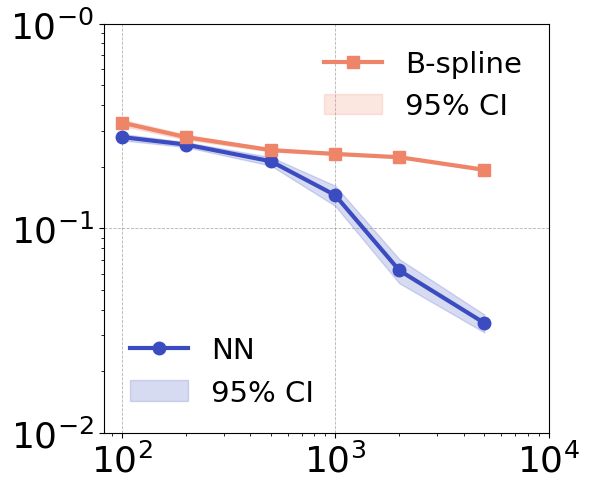}
        \label{fig:sub2}
    \end{subfigure}

    \caption{Comparison of convergence rates for $ d = 1 $ (left) and $ d = 2 $ (right). }
    \label{fig:compare-convergence}
\end{figure}

\begin{figure*}[t]
\centering
\begin{subfigure}{0.2\textwidth}
\includegraphics[width=1\linewidth]{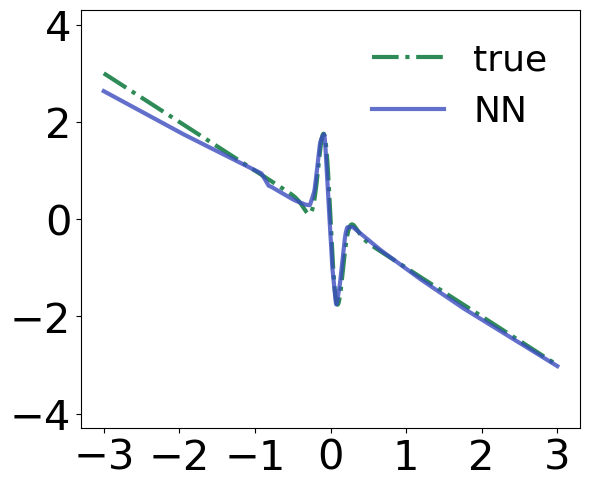}
\caption{NN estimator, 1D}
\label{fig:sub1}
\end{subfigure}
\begin{subfigure}{0.2\textwidth}
\includegraphics[width=1\linewidth]{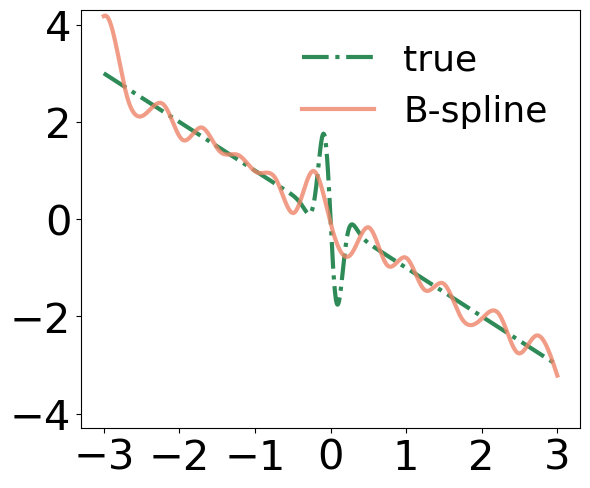}
\caption{$B$-spline estimator, 1D}
\label{fig:sub2}
\end{subfigure}
\begin{subfigure}{0.19\textwidth}     
\includegraphics[width=1\linewidth]{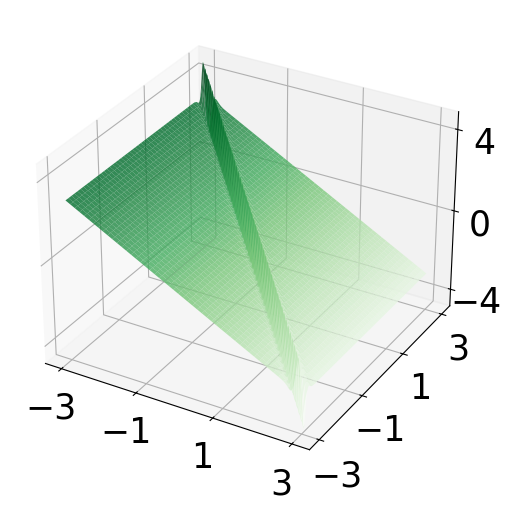}
\caption{True function, 2D}
\label{fig:sub3}
\end{subfigure}
\begin{subfigure}{0.19\textwidth}  
\includegraphics[width=1\linewidth]{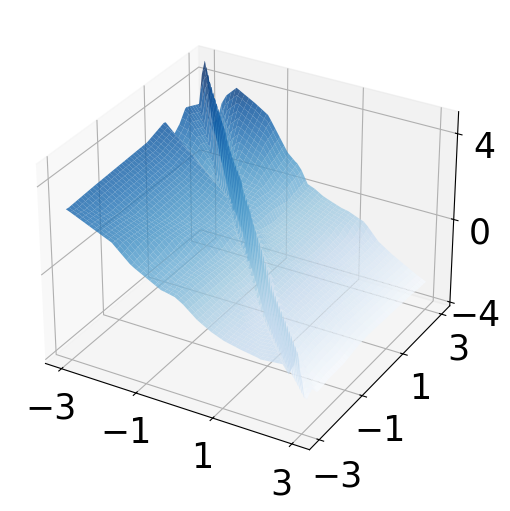}
\caption{NN estimator, 2D}
\label{fig:sub4}
\end{subfigure}
\begin{subfigure}{0.19\textwidth}  
\includegraphics[width=1\linewidth]{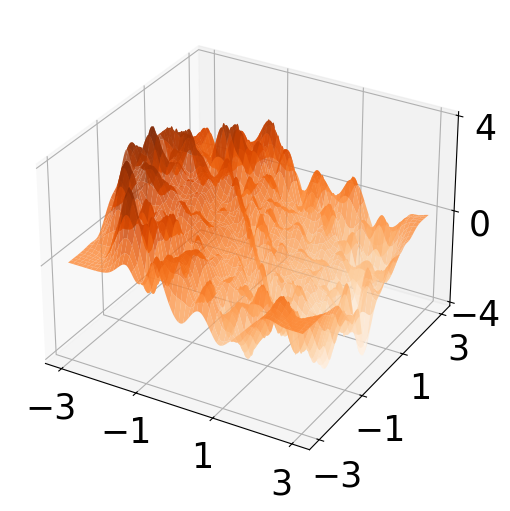}
\caption{$B$-spline estimator, 2D}
\label{fig:sub5}
\end{subfigure}
\caption{Comparison of the ability to capture local fluctuations (N=5000).}
\label{fig:capture-fluctuations}
\end{figure*}

\subsection{Conclusion}

In summary, Figures~\ref{fig:compare-convergence} and~\ref{fig:capture-fluctuations} clearly demonstrate that our neural network-based estimator outperforms the $B$-spline-based estimator both in terms of convergence rate and its ability to capture local fluctuations.

Furthermore, for the $B$-spline-based estimator, the number of basis functions grows exponentially with the dimension $ d $, leading to a substantial increase in memory requirements. For example, when $ d = 5 $, assuming 8 basis functions per dimension, the resulting coefficient matrix has size $ MN \times 32768 $. If each entry is stored as a 64-bit float (8 bytes), this amounts to over 24GB of memory for $M=100$ and $N = 1000 $, not including the additional memory required for matrix inversion. As a result, the $B$-spline estimator becomes computationally infeasible on a standard laptop even for moderate values of $d$, for example $d\geq 3$. In contrast, the neural network estimator only requires storing the network parameters and minibatch data during training, making it far more scalable in high-dimensional settings.



\section{Proof Sketch}\label{sec:proof-sketch}

The proof of Theorem \ref{thm:main} consists of two steps.  In the first step (see  Proposition~\ref{prop:test-train-error-comparaison}), we establish an upper bound of the risk~$\mathcal{R}(\hat{f}, f_0)$ defined in \eqref{eq:test-error} by using the following expected empirical train risk~$\hat{\mathcal{R}}_{\mathcal{D}_N}(\hat{f}, f_0)$, defined by
\begin{align}\label{eq:def-empirical-train-error}
   &\hat\calR_{\mathcal{D}_N}(\hat{f}, f_0)\nonumber\\
   &\coloneqq \mathbb{E}\left[\frac{1}{NM}\sum_{n=1}^{N} \sum_{m=0}^{M-1}(\hat{f}(\bar{X}_{t_m}^{(n)})-f_{0}(\bar{X}_{t_m}^{(n)}))^2\right],
\end{align}
and an estimation of the covering number of the neural network estimator set $\mathcal{F}(L, \mathbf{p}, s, F)$.   Recall that for any $\delta>0$, a subset $\mathcal{G}\subset \mathcal{F}$ is called a $\delta$-net of $\mathcal{F}$ with respect to some norm $\Vert \cdot \Vert $ on $\mathcal{F}$, if for every $f\in\calF$, there exists $g\in \calG$ such that $\Vert f-g \Vert \leq \delta$. The minimal cardinality of such $\delta$-net net is called the covering number, denoted by $\mathfrak{N}(\delta, \mathcal{F}, \Vert \cdot \Vert)$ (see e.g. \citet{vanderVaart2023}).  When there is no ambiguity about the training set, we simplify the notation by writing $\hat{\mathcal{R}}(\hat{f}, f_0)$ instead of $\hat{\mathcal{R}}_{\mathcal{D}_N}(\hat{f}, f_0)$. 
In the second step (see Proposition~\ref{prop:upper-bound-train-error}), we derive an upper bound on the expected empirical training risk~$\hat{\mathcal{R}}(\hat{f}, f_0)$ in terms of $\Psi^{\calF}(\hat{f})$, the approximation error $\inf_{f \in \mathcal{F}} \Vert f - f_0 \Vert_{\infty}^2$, the neural network parameters $(s, L, F)$, and the data parameters $\Delta$, $N$, and $d$.

\subsection{Upper Bound of $\mathcal{R}(\hat{f}, f_0)$}

This subsection is devoted to establishing the following proposition.

\begin{prop}\label{prop:test-train-error-comparaison}
Suppose that Assumptions \ref{assum:lip} and \ref{assum:parameter} hold. Fix an arbitrary $\bar\varepsilon \in (0,1)$, we have 
\begin{align}
&\mathcal{R}(\hat{f},f_{0})
\leq 3\bar\varepsilon+2     \hat\calR(\hat{f}, f_0), \nonumber\\
& +\frac{44F^2\log\big(Cs(L\log s+\log d+\log 4F-\log \bar\varepsilon)\big)}{N},\nonumber
\end{align}
where  $C>0$  is a universal constant.
\end{prop}

The proof of Proposition \ref{prop:test-train-error-comparaison} relies on the following two lemmas. 

\begin{lem}[Lemma A.2 in \citet{Denis2021}]
\label{lem:A-2-denis}
Let  $X_{1}, \dots, X_{N}$  be independent copies of a random variable $ X \in \mathcal{X}$. Let  $\mathcal{G}$ be a class of real-valued functions on $\mathcal{X}$ . For each $g \in \mathcal{G}$, and  $x \in \mathcal{X}$, we assume that $0 \leq g(x) \leq L$, with $L>0$. We consider $\mathcal{G}_{\delta}$ an $\delta$-net of $\mathcal{G} $ w.r.t. $\|\cdot\|_{\sup}$ and we denote by $\mathfrak{N}(\delta, \mathcal{G}, \Vert \cdot \Vert_{\sup})$ its cardinality. Then, the following holds
\begin{align}
&\mathbb{E}\left[\sup _{g \in \mathcal{G}}\left(\mathbb{E}[g(X)]-\frac{2}{N} \sum_{i=1}^{N} g\left(X_{i}\right)\right)\right] &\nonumber\\
&\qquad \leq 3 \varepsilon+\frac{11 L \log \big(\mathfrak{N}(\delta, \mathcal{G}, \Vert \cdot \Vert_{\sup})\big)}{N}.\nonumber
\end{align}
\end{lem}
\begin{lem}[Lemma 4.13 in \citet{Oga2024}]
\label{lem:lemma413-in-oga}
Let  $\mathcal{F} \subset \mathcal{F}(L, \mathbf{p}, s)$ . If $ s \geq 2$ , we have for all  $\delta \in(0,1) $
$$
\log \mathfrak{N}\left(\delta, \mathcal{F},\|\cdot\|_{\infty}\right) \leq C s(L \log s+\log d-\log \delta)
$$
where  $C>0$  is a universal constant.
\end{lem}

\begin{proof}[Proof of Proposition \ref{prop:test-train-error-comparaison}]

In what follows, we adopt the notation $X_{t_0: t_{M-1}}$ for $(X_{t_0}, \dots, X_{t_{M-1}} )$ and $   \bar{X}^{(n)}_{t_0 : t_{M-1}}$ for $ \big( \bar X^{(n)}_{t_0},\dots, \bar X^{(n)}_{t_{M-1}} \big)$. For a fixed $\bar f\in\nnset$, we define the function $g_{\bar{f}}$ from $(\RR^d)^M$ to $\RR$  by 
\begin{align}
    & x\!=\!(x_0,\dots,x_{M-1})\mapsto\nonumber\\
    &\qquad g_{\bar{f}}(x)=\!\frac{1}{M}\!\sum_{m=0}^{M-1}\big(\bar{f}(x_m)-{f}_{0}(x_m)\big)^2,\nonumber
\end{align}
which is bounded above by $4F^2$ by Assumption \ref{assum:parameter} and supported on $([0,1]^d)^M$ since both ${f}_0$ and $\bar{f}$ are supported on $[0,1]^d$.  Moreover, we define 
\begin{equation}
\calG\coloneqq \{g_{\bar{f}}:  \bar{f}\in\nnset \},\nonumber
\end{equation}
equipped with the norm $\Vert g\Vert_{\sup}\coloneqq \sup_{x\in([0,1]^d)^M}|g(x)|$. 

The estimation risk $\mathcal{R}(\hat{f}, f_{0})$ can be decomposed by
\begin{align}
\mathcal{R}(\hat{f},f_{0})
\leq &\EE \left[ g_{\hat{f}}(X_{t_0:t_{M-1}})-\frac{2}{N}\sum_{n=1}^{N}g_{\hat{f}}\big(\bar{X}_{t_0:t_{M-1}}^{(n)}\big)\right]  \nonumber\\
&\quad + 2     \hat\calR(\hat{f}, f_0),\nonumber
\end{align}
where the first term can be upper bounded as follows
\begin{align}
&\EE \left[ g_{\hat{f}}(X_{t_0:t_{M-1}})-\frac{2}{N}\sum_{n=1}^{N}g_{\hat{f}}\big(\bar{X}_{t_0:t_{M-1}}^{(n)}\big)\right] \nonumber\\
&\leq \EE \left[ \sup_{g_{\bar{f}}\in \calG}\left(g_{\bar{f}}(X_{t_0:t_{M-1}})-\frac{2}{N}\sum_{n=1}^{N}g_{\bar{f}}\big(\bar{X}_{t_0:t_{M-1}}^{(n)}\big)\right)\right] \nonumber\\
&\leq 3\varepsilon+\frac{44F^2\log\big(\mathfrak{N}(\varepsilon,\mathcal{G}, \Vert \cdot \Vert_{\sup})\big)}{N},\nonumber
\end{align}
for any $\varepsilon\in(0,1)$ by applying Lemma \ref{lem:A-2-denis}.

We now provide an  upper bound on  the covering number $\mathfrak{N}(\varepsilon, \mathcal{G}, \| \cdot \|_{\sup})$ using the covering number over $\nnset$. 
Let $\mathfrak{N}_{\frac{\varepsilon}{4F}}\!\coloneqq \mathfrak{N}(\frac{\varepsilon}{4F}, \nnset, \| \cdot \|_{\infty})$ and 
let $\mathcal{F}_{\text{net}}\coloneqq\{f_1, ..., f_{\mathfrak{N}_{\frac{\varepsilon}{4F}}}\}$ be an $\frac{\varepsilon}{4F}$-net of $\nnset$. Then for every $\bar f\in\nnset$, there exists $\bar f_{\frac{\varepsilon}{4F}}\in \mathcal{F}_{\text{net}}$ such that $\Vert \bar f- \bar f_{\frac{\varepsilon}{4F}}\Vert_{\infty}\leq \frac{\varepsilon}{4F}$. It follows that 
\begin{align}
&\Vert g_{\bar{f}}-g_{\bar f_{\frac{\varepsilon}{4F}}} \Vert_{\infty}\leq \Vert \bar{f}-\bar{f}_{\frac{\varepsilon}{4F}}\Vert_{\infty}\Vert \bar{f}+\bar{f}_{\frac{\varepsilon}{4F}}-2f_0\Vert_{\infty}\leq \varepsilon.\nonumber
\end{align}
Therefore, Lemma \ref{lem:lemma413-in-oga} implies that
\begin{align*}
&\log\big(\mathfrak{N}(\varepsilon, \mathcal{G}, \| \cdot \|_{\sup})\big)\nonumber\\
& \quad \leq \log \big(\mathfrak{N}(\frac{\varepsilon}{4F}, \nnset, \| \cdot \|_{\infty})\big) \\
& \quad \leq Cs(L\log s+\log d+\log 4F-\log \varepsilon).
\end{align*}
Combining the results above, we get
\begin{align*}
&\mathcal{R}(\hat{f},f_{0}) 
\leq 3\varepsilon+2\hat{\mathcal{R}}(\hat{f},f_{0})\\
& \quad + \frac{44F^2Cs(L\log s+\log d+\log 4F-\log \varepsilon)}{N}.\hfill\qedhere 
\end{align*}
\end{proof}

\subsection{Upper Bound of $\hat{\mathcal{R}}(\hat{f}, f_0)$}
This section aims to establish an upper bound on the expected empirical training risk $\hat{\mathcal{R}}(\hat{f},f_{0})$. We abbreviate $\mathcal{F}$ for the neural network function class $\nnset$, and $\mathfrak{N}_\delta$ for the covering number $\mathfrak{N}(\delta, \nnset, \|\cdot\|_{\infty})$.
\begin{prop}\label{prop:upper-bound-train-error}
Suppose that Assumptions \ref{assum:lip} and \ref{assum:parameter} hold.
There exists a constant $\mathfrak{C}'$ depending on $C_b, C_{\sigma}, L_{b}, L_{\sigma}, T$ and the universal constant $C$ in Lemma \ref{lem:lemma413-in-oga} such that 
\begin{align}
&\hat{\calR}(\hat{f}, f_0)\leq 2\Psi^{\calF}(\hat{f})+ 3\inf_{f\in\mathcal{F}}\calR(f,f_0)\nonumber\\
&+\mathfrak{C}'\left(F^2 \Delta +\frac{s(L \log s + \log d)+F}{N}+s\frac{\log N}{N}\right).\nonumber
\end{align}
\end{prop}
To prove Proposition \ref{prop:upper-bound-train-error}, we first note that $Y_{t_m}^{(n)}$, defined in \eqref{eq:defY}, admits the following decomposition :
\begin{align}\label{eq:decompos-Y}
Y^{(n)}_{t_m}= &\, b^{i}(\bar{X}^{(n)}_{t_m}) \nonumber\\
&+ \frac{1}{\Delta}\int_{t_m}^{t_{m+1}}(b^i(\bar{X}^{(n)}_s)-b^i(\bar{X}^{(n)}_{t_m}))\dd s \nonumber\\
&+ \frac{1}{\Delta}\int_{t_m}^{t_{m+1}}\sigma^i(\bar{X}^{(n)}_{s})\dd B_{s}^{(n)} \nonumber\\
\eqqcolon &\, b^{i}(\bar{X}^{(n)}_{t_m})+ I^{(n)}_{t_m} + {\Sigma}^{(n)}_{t_m}, 
\end{align}
where $B^{(n)}, 1\leq n\leq N$ are i.i.d.  $d$-dimensional standard Brownian motion. 
Given any two functions $f_1, f_2$ in $\nnset$, we define 
\[\Psi_{\mathcal{D}_N}(f_1,f_2)\coloneqq \EE \left[ \calQ_{\mathcal{D}_N}(f_1)- \calQ_{\mathcal{D}_N}(f_2)\right]\] 
with $\calQ_{\mathcal{D}_N}$ defined in \eqref{eq:train-nn-loss}. Let $\bar{f}$ be an arbitrary function in $\calF$. A direct computation yields the following decomposition of $\hat{\mathcal{R}}(\hat{f}, f_0)$ :
\begin{align}\label{eq:decomposition-train-error}
&\hat{\mathcal{R}}(\hat{f},f_{0})  =\Psi_{\mathcal{D}_N}(\hat{f},\bar{f})\\
& \quad +\mathbb{E}\left[\frac{1}{NM}\sum_{n=1}^{N}\sum_{m=0}^{M-1}(\bar{f}(\bar{X}_{t_m}^{(n)})-f_{0}(\bar{X}_{t_m}^{(n)}))^2\right]\nonumber \\
& \quad + 2\mathbb{E}\left[\frac{1}{NM}\sum_{n=1}^{N}\sum_{m=0}^{M-1}I_{t_m}^{(n)}\Big(\hat{f}(\bar{X}_{t_m}^{(n)})-\bar{f}(\bar{X}_{t_m}^{(n)})\Big)\right].\nonumber\\
& \quad + 2\mathbb{E}\left[\frac{1}{NM}\sum_{n=1}^{N}\sum_{m=0}^{M-1}\Sigma_{t_m}^{(n)}\Big(\hat{f}(\bar{X}_{t_m}^{(n)})-{f}_0(\bar{X}_{t_m}^{(n)})\Big)\right],\nonumber
\end{align}
where the last term follows from the fact that $\bar{f}$ and $f_0$ are deterministic (i.e., independent of the training set), and $\EE [\Sigma_{t_m}^{(n)}]=0$ and independent of $\bar{X}_{t_m}^{(n)}$ for each $m$ and $n$. 

Let $\mathcal{F}_{\text{net}} \coloneqq \{f_1, \dots, f_{\mathfrak{N}_{\delta}}\}$ denote a $\delta$-net of $\calF$ with respect to $\Vert \cdot \Vert _{\infty}$, Then, there exists a random index $\mathfrak{n}^*$ valued in $\{1, ..., \mathfrak{N}_{\delta}\}$ such that $|\hat f - f_{\mathfrak{n}^*}|\leq \delta$. A reasoning similar to \citet[Lemma 4.8, 4.11, 4.12]{Oga2024} leads to the following lemma.
\begin{lem}\label{lem:i-and-sigma-bound} Suppose that Assumptions \ref{assum:lip} and \ref{assum:parameter} hold. Let $C_{\star}=4 (C_b^2 + C_{\sigma}^2)\exp(4(L_b^2+L_{\sigma}^2))$. 
\begin{enumerate}
\item For every $\bar f \in \calF$ and  $\varepsilon \in (0,1)$,
\begin{align}\label{eq:ineq-I-term}
&\mathbb{E}\left[\frac{1}{NM}\sum_{n=1}^{N}\sum_{m=0}^{M-1}I_{t_m}^{(n)}\Big(\hat{f}(\bar{X}_{t_m}^{(n)})-\bar{f}(\bar{X}_{t_m}^{(n)})\Big)\right]\nonumber\\
&\leq \frac{\varepsilon}{4}\hat{\calR}(\hat{f}, f_0)+\frac{\varepsilon}{2}\hat{\calR}(\bar{f}, f_0)+ \frac{3L_b^2C_{\star}}{4\varepsilon}\Delta.
\end{align}
\item  There exists a constant $\bar{C}_{\sigma}$ depending only on $C_{\sigma}$ such that for every $n\in\{1, ..., N\}$,
\begin{align}
&\EE \left[\frac{1}{MN}\sum_{n=1}^{N}\sum_{m=0}^{M-1}\Sigma_{t_m}^{(n)}\big(\hat{f}-f_{\mathfrak{n}^*}\big)(\bar{X}_{t_m}^{(n)})\right]\nonumber\\
& \leq \delta \sqrt{\frac{1}{2}L_{\sigma}^2 C_{\star}}+\frac{\bar{C}_{\sigma}}{\sqrt{T}}\int_{0}^{\delta}\sqrt{\log \mathfrak{N}_{u}}\,\dd u.\label{eq:ineq-sigma-term}
\end{align}

\end{enumerate}
\end{lem}
Furthermore, the proof of Proposition~\ref{prop:upper-bound-train-error} relies on the following lemma.
\begin{lem}\label{lem:main-diff-from-oga} Suppose that Assumptions \ref{assum:lip} and \ref{assum:parameter} hold. For every $\varepsilon\in(0,1)$, we have 
\begin{align}
 &\mathbb{E}\left[\frac{1}{NM}\sum_{n=1}^{N}\sum_{m=0}^{M-1}\Sigma_{t_m}^{(n)}\big({f_{\mathfrak{n}^*}}(\bar{X}_{t_m}^{(n)})-{f}_0(\bar{X}_{t_m}^{(n)})\big)\right]\nonumber\\
 &\leq \frac{\varepsilon}{4}\hat{\mathcal{R}}(\hat{f},f_{0}) + \gamma_{\varepsilon},\nonumber
\end{align} 
where $\gamma_{\varepsilon}=\varepsilon F\delta + \varepsilon \frac{F^2L_{\sigma}^2 C_{\star}}{C_{\sigma}^2}\Delta+\frac{4C_{\sigma}^2}{\varepsilon TN}\big(\log 2 +2\log\mathfrak{N}_{\delta}\big)$.
\end{lem}

\begin{proof}[Proof sketch of Lemma \ref{lem:main-diff-from-oga}]
For a fixed function $f\in\calF$ and for a fixed $n$, we define the processes $\widehat{M}^{(n)}(f)$,  $\widehat{A}^{(n)}(f)$, $\overline{M}(f)$ and $ \bar{A}(f)$  by 
\begin{align}
&\widehat{M}^{(n)}(f)_s\coloneqq  \sum_{m=0}^{M-1}(f-f_0)(\bar{X}_{t_m}^{(n)})\int_{s\wedge t_m}^{s\wedge t_{m+1}}\!\!\!\sigma ^i\big(\bar X_u^{(n)}\big)\dd B_u^{(n)},\nonumber\\
&\overline{M}(f)_s\coloneqq \frac{1}{N}\sum_{n=1}^{N} \widehat{M}^{(n)}(f),\nonumber\\
&\widehat{A}^{(n)}(f)\coloneqq \langle\widehat{M}^{(n)}(f)\rangle,\nonumber\\
&  \bar{A}(f)\coloneqq \big\langle \overline{M}(f)\big\rangle=\frac{1}{N^2}\sum_{n=1}^{N}\widehat{A}^{(n)}(f). \nonumber
\end{align}
It follows that 
\begin{align}
 &\mathbb{E}\left[\frac{1}{NM}\sum_{n=1}^{N}\sum_{m=0}^{M-1}\Sigma_{t_m}^{(n)}(f_{\mathfrak{n}^*}(\bar{X}_{t_m}^{(n)})-{f}_0(\bar{X}_{t_m}^{(n)}))\right]\nonumber\\
 & \leq \mathbb{E}\left[\left|\frac{1}{T} \overline{M}(f_{\mathfrak{n}^*})_{T}\right|\right]=\tfrac{2}{T}\EE \left[\left|\overline{\xi}_{f_{\mathfrak{n}^*}}\right| \left(\bar{A}(f_{\mathfrak{n}^*})_T+\overline{D}_{f_{\mathfrak{n}^*}} \right)^{\frac{1}{2}}\right]\nonumber\\
 & \leq \tfrac{2}{T} \sqrt{\EE \left[\overline{\xi}_{f_{\mathfrak{n}^*}}^2\right] \left(2\EE \left[\bar{A}(f_{\mathfrak{n}^*})_T\right]+\varepsilon'\right)},\label{eq:decomposition-sigma-n}
\end{align}
where for every $f\in\calF$,  we denote by $\overline{\xi}_{f}\coloneqq \frac{\overline{M}(f)_T}{2\sqrt{\bar{A}(f)_T+\overline{D}_f}}$
and $\overline{D}_f=\EE\left[ \bar{A}(f)_T\right]+\varepsilon'$ for some $\varepsilon'>0$.
As $\overline{M}(f)$ is a continuous local martingale for every $f \in\calF$, it follows from \citet[Lemma 4.9]{Oga2024} that
\begin{align}
    \EE \left[\tfrac{\sqrt{\overline{D}_{f}}}{\sqrt{\bar{A}(f_{n})_T+\overline{D}_{f}}}\exp\left\{ 2\overline{\xi}_{f}^2\right\}\right]\leq 1.
\end{align}
Hence, an argument analogous to \citet[Lemma 4.10]{Oga2024} yields $\EE \left[ \overline{\xi}_{f_{n^{*}}}^2\right]\leq \frac{1}{4}\log 2 +\frac{1}{2}\log\mathfrak{N}_{\delta}$. 

For $\EE \left[\bar{A}(f_{\mathfrak{n}^*})_T\right]$, we have
\begin{align}
    \EE \left[\bar{A}(f_{\mathfrak{n}^*})_T\right]
    &=\frac{1}{N}\cdot \EE \Big[\frac{1}{N}\sum_{n=1}^{N}\widehat{A}^{(n)}(f_{\mathfrak{n}^*})_T\Big],\nonumber 
\end{align}
and an argument analogous to \citet[Lemma 4.10]{Oga2024} yields
\begin{align}
    &\EE \left[\frac{1}{N}\sum_{n=1}^{N}\widehat{A}^{(n)}(f_{\mathfrak{n}^*})_T\right]\nonumber\\
    & \leq 2 C_{\sigma}^2T\big(\hat{\mathcal{R}}(\hat{f},f_{0})+4F\delta\big) + 8F^2L_{\sigma}^2 C_{\star}T\Delta.\nonumber
\end{align}
Finally, inserting the above bounds into \eqref{eq:decomposition-sigma-n} yields
\begin{align}
 &\mathbb{E}\left[\frac{1}{NM}\sum_{n=1}^{N}\sum_{m=0}^{M-1}\Sigma_{t_m}^{(n)}({f_{\mathfrak{n}^*}}(\bar{X}_{t_m}^{(n)})-{f}_0(\bar{X}_{t_m}^{(n)}))\right] \nonumber\\
 &\leq \frac{1}{T} \sqrt{\log 2 +2\log\mathfrak{N}_{\delta}}\nonumber\\
 &\: \cdot \sqrt{\frac{4}{N}\left[   C_{\sigma}^2T \big(\hat{\mathcal{R}}(\hat{f},f_{0})+4F\delta\big) + 4F^2L_{\sigma}^2 C_{\star}T\Delta \right]+\varepsilon'}.\nonumber
\end{align}
By letting $\varepsilon'\rightarrow0$ and by applying the AM-GM inequality $\sqrt{xy}\leq \frac{\varepsilon}{4}x+\frac{1}{\varepsilon}y$, we get the desired inequality. 
\end{proof}
\begin{proof}[Proof sketch of Proposition \ref{prop:upper-bound-train-error}]
By inserting the results of Lemmas \ref{lem:i-and-sigma-bound} and \ref{lem:main-diff-from-oga} into \eqref{eq:decomposition-train-error}, 
we obtain, for a  an arbitrary function $\bar{f}$  in $\calF$, for $\varepsilon\in(0,1)$,
\begin{align}
\hat{\calR}(\hat{f}, f_0)\leq \frac{\Psi^{\calF}(\hat{f})}{1-\varepsilon}+ \frac{1+\varepsilon}{1-\varepsilon}{\calR}(\bar f, f_0)+\frac{\gamma'_{\varepsilon}}{1-\varepsilon},\nonumber
\end{align}
where we apply  $\hat{\calR}(\bar f, f_0)={\calR}(\bar f, f_0)$ and 
\begin{align}
    \gamma'_{\varepsilon}\!=&\frac{3L_b^2C_{\star}}{2\varepsilon}\Delta+\delta\sqrt{2L_{\sigma}^2C_{\star}}+\frac{2\bar{C}_{\sigma}}{\sqrt{T}}\int_{0}^{\delta}\!\!\!\sqrt{\log \mathfrak{N}_{u}}\,\dd u\nonumber\\
    &+2\varepsilon F\delta + 2\varepsilon \frac{F^2L_{\sigma}^2 C_{\star}}{C_{\sigma}^2}\Delta+\frac{8C_{\sigma}^2}{\varepsilon TN}\big(\log 2 +2\log\mathfrak{N}_{\delta}\big).\nonumber
\end{align}
By applying Lemma \ref{lem:lemma413-in-oga}, we have
\begin{align}
\int_{0}^{\delta}\sqrt{\log \mathfrak{N}_{u}}\,\dd u\leq &\int_{0}^{\delta} \sqrt{C s(L \log s+\log d)}\,\dd u \nonumber\\
& +\sqrt{Cs}\int_{0}^{\delta}\sqrt{-\log u }\,\dd u. \nonumber
\end{align}
Using the change of variable $t=-\log u$ and the inequality  $\Gamma\big(\tfrac{3}{2}, -\log (\delta)\big)\leq \tfrac{3}{4}\delta \sqrt{-\log \delta}$ valid for $\delta < \exp(-\frac{1}{2})$, we set $\varepsilon=\frac{1}{2}$ and $\delta=\frac{1}{N}$ to obtain \begin{align}
&\hat{\calR}(\hat{f}, f_0)\leq 2\Psi^{\calF}(\hat{f})+ 3{\calR}(\bar f, f_0)\nonumber\\
&\qquad +2L_{\sigma}^2C_{\star}\big(3+\tfrac{F^2}{C_{\sigma^2}}\big)\Delta+ C_1\frac{1}{N}+C_2\frac{\log N}{N},\nonumber 
\end{align}
with
\begin{align}
    C_1= & 2\sqrt{2L_\sigma^2 C_{\star}}+\frac{4\bar{C}_{\sigma}}{\sqrt{T}}\sqrt{Cs(L \log s+\log d)}+2F\nonumber\\
   & +\frac{32 C_{\sigma}^2}{T}\big( \log 2 +2 Cs (L \log s+\log d)\big)\nonumber
\end{align}
and $  C_2= \frac{6\bar{C}_{\sigma}\sqrt{Cs}}{\sqrt{T}} +\frac{64 C_{\sigma}^2Cs}{T}$. We conclude by noting that ${\calR}(\bar f, f_0) \leq \Vert \bar f - f_0 \Vert_\infty$, and that the above inequality holds for any arbitrary function $\bar{f} \in \calF$.
\end{proof}

\begin{proof}[Proof  of Theorem \ref{thm:main}]
Theorem~\ref{thm:main} follows directly from Propositions~\ref{prop:test-train-error-comparaison} and~\ref{prop:upper-bound-train-error} by setting $\bar\varepsilon=\frac{\log N}{N}$. 
\end{proof}

\begin{proof}[Proof of Corollary \ref{cor:cor-of-thm-main}]
Theorem 1 in \citet{SchmidtHieber2020} provides an upper bound of $\inf_{f\in\nnset}\calR(f,f_0)$.  Corollary~\ref{cor:cor-of-thm-main} then follows directly from Theorem~\ref{thm:main} and Theorem 1 in \citet{SchmidtHieber2020}.
\end{proof}


\section*{Acknowledgments}
This work was supported by the French National Research Agency (ANR) under the France 2030 program, reference “ANR-23-EXMA-0011” (MIRTE Project), and originates from the Master’s thesis of Yuzhen Zhao, completed during her internship at the LEDa laboratory, Université Paris Dauphine–PSL, under the supervision of Yating Liu.

\bibliography{aaai2026}

@book {gyorfi2002distribution,
    AUTHOR = {Gy\"{o}rfi, L\'{a}szl\'{o} and Kohler, Michael and Krzy\.{z}ak, Adam and Walk,
              Harro},
     TITLE = {A distribution-free theory of nonparametric regression},
    SERIES = {Springer Series in Statistics},
 PUBLISHER = {Springer-Verlag, New York},
      YEAR = {2002},
     PAGES = {xvi+647},
      ISBN = {0-387-95441-4},
   MRCLASS = {62-02 (62G08 62H05 62M10 62M45 62Nxx)},
  MRNUMBER = {1920390},
MRREVIEWER = {Ewaryst Rafaj\l owicz},
       DOI = {10.1007/b97848},
       URL = {https://doi.org/10.1007/b97848},
}

@INPROCEEDINGS{8516991,
  author={Yildiz, Cagatay and Heinonen, Markus and Intosalmi, Jukka and Mannerstrom, Henrik and Lahdesmaki, Harri},
  booktitle={2018 IEEE 28th International Workshop on Machine Learning for Signal Processing (MLSP)}, 
  title={Learning Stochastic Differential Equations with Gaussian Processes without Gradient Matching}, 
  year={2018},
  volume={},
  number={},
  pages={1-6},
  doi={10.1109/MLSP.2018.8516991}}

@article{Frishman2020,
  title = {Learning Force Fields from Stochastic Trajectories},
  author = {Frishman, Anna and Ronceray, Pierre},
  journal = {Phys. Rev. X},
  volume = {10},
  issue = {2},
  pages = {021009},
  numpages = {32},
  year = {2020},
  month = {Apr},
  publisher = {American Physical Society},
  doi = {10.1103/PhysRevX.10.021009},
  url = {https://link.aps.org/doi/10.1103/PhysRevX.10.021009}
}

@article{gao2024learning,
  title={Learning interpretable dynamics of stochastic complex systems from experimental data},
  author={Gao, Ting-Ting and Barzel, Baruch and Yan, Gang},
  journal={Nature communications},
  volume={15},
  number={1},
  pages={6029},
  year={2024},
  publisher={Nature Publishing Group UK London}
}

@article{BAE2025116440,
title = {Inferring the Langevin equation with uncertainty via Bayesian neural networks},
journal = {Chaos, Solitons \& Fractals},
volume = {197},
pages = {116440},
year = {2025},
issn = {0960-0779},
doi = {https://doi.org/10.1016/j.chaos.2025.116440},
url = {https://www.sciencedirect.com/science/article/pii/S0960077925004539},
author = {Youngkyoung Bae and Seungwoong Ha and Hawoong Jeong}
}

@article {Dalalyan2005,
    AUTHOR = {Dalalyan, Arnak},
     TITLE = {Sharp adaptive estimation of the drift function for ergodic
              diffusions},
   JOURNAL = {Ann. Statist.},
  FJOURNAL = {The Annals of Statistics},
    VOLUME = {33},
      YEAR = {2005},
    NUMBER = {6},
     PAGES = {2507--2528},
      ISSN = {0090-5364},
   MRCLASS = {62M05 (62G07 62G20)},
  MRNUMBER = {2253093},
MRREVIEWER = {Ursula U. M\"{u}ller},
       DOI = {10.1214/009053605000000615},
       URL = {https://doi.org/10.1214/009053605000000615},
}

@article {Comte2020,
    AUTHOR = {Comte, Fabienne and Genon-Catalot, Valentine},
     TITLE = {Nonparametric drift estimation for i.i.d. paths of stochastic
              differential equations},
   JOURNAL = {Ann. Statist.},
  FJOURNAL = {The Annals of Statistics},
    VOLUME = {48},
      YEAR = {2020},
    NUMBER = {6},
     PAGES = {3336--3365},
      ISSN = {0090-5364},
   MRCLASS = {62G07 (60J60 62M05)},
  MRNUMBER = {4185811},
       DOI = {10.1214/19-AOS1933},
       URL = {https://doi.org/10.1214/19-AOS1933},
}

@book {Kutoyants2004,
    AUTHOR = {Kutoyants, Yury A.},
     TITLE = {Statistical inference for ergodic diffusion processes},
    SERIES = {Springer Series in Statistics},
 PUBLISHER = {Springer-Verlag London, Ltd., London},
      YEAR = {2004},
     PAGES = {xiv+481},
      ISBN = {1-85233-759-1},
   MRCLASS = {62-02 (60J60 62F10 62G05 62M02 62M05)},
  MRNUMBER = {2144185},
MRREVIEWER = {Stefano Maria Iacus},
       DOI = {10.1007/978-1-4471-3866-2},
       URL = {https://doi.org/10.1007/978-1-4471-3866-2},
}

@inproceedings{ren2024statistical,
  title={Statistical spatially inhomogeneous diffusion inference},
  author={Ren, Yinuo and Lu, Yiping and Ying, Lexing and Rotskoff, Grant M},
  booktitle={Proceedings of the AAAI Conference on Artificial Intelligence},
  volume={38},
  number={13},
  pages={14820--14828},
  year={2024}
}

@book {Bressloff2014,
    AUTHOR = {Bressloff, Paul C.},
     TITLE = {Stochastic processes in cell biology},
    SERIES = {Interdisciplinary Applied Mathematics},
    VOLUME = {41},
 PUBLISHER = {Springer, Cham},
      YEAR = {2014},
     PAGES = {xviii+679},
      ISBN = {978-3-319-08487-9; 978-3-319-08488-6},
   MRCLASS = {92C37 (60H30 82C31 82C70 92-01 92C40)},
  MRNUMBER = {3244328},
MRREVIEWER = {John Adam},
       DOI = {10.1007/978-3-319-08488-6},
       URL = {https://doi.org/10.1007/978-3-319-08488-6},
}

@article{Yarotsky2017,
title = {Error bounds for approximations with deep ReLU networks},
journal = {Neural Networks},
volume = {94},
pages = {103-114},
year = {2017},
issn = {0893-6080},
doi = {https://doi.org/10.1016/j.neunet.2017.07.002},
url = {https://www.sciencedirect.com/science/article/pii/S0893608017301545},
author = {Dmitry Yarotsky},
}

@book {Gardiner2004,
    AUTHOR = {Gardiner, C. W.},
     TITLE = {Handbook of stochastic methods for physics, chemistry and the
              natural sciences},
    SERIES = {Springer Series in Synergetics},
    VOLUME = {13},
   EDITION = {Third},
 PUBLISHER = {Springer-Verlag, Berlin},
      YEAR = {2004},
     PAGES = {xviii+415},
      ISBN = {3-540-20882-8},
   MRCLASS = {00A69 (60-01 60Hxx 60Jxx 82C31)},
  MRNUMBER = {2053476},
       DOI = {10.1007/978-3-662-05389-8},
       URL = {https://doi.org/10.1007/978-3-662-05389-8},
}

@book {vanderVaart2023,
    AUTHOR = {van der Vaart, A. W. and Wellner, Jon A.},
     TITLE = {Weak convergence and empirical processes---with applications
              to statistics},
    SERIES = {Springer Series in Statistics},
      NOTE = {Second edition [of  1385671]},
 PUBLISHER = {Springer, Cham},
      YEAR = {2023},
     PAGES = {xvii+679},
      ISBN = {978-3-031-29038-1; 978-3-031-29040-4},
   MRCLASS = {60-02 (60B12 60F05 62G30)},
  MRNUMBER = {4628026},
       DOI = {10.1007/978-3-031-29040-4},
       URL = {https://doi.org/10.1007/978-3-031-29040-4},
}

@book {Karatzas1998,
    AUTHOR = {Karatzas, Ioannis and Shreve, Steven E.},
     TITLE = {Methods of mathematical finance},
    SERIES = {Applications of Mathematics (New York)},
    VOLUME = {39},
 PUBLISHER = {Springer-Verlag, New York},
      YEAR = {1998},
     PAGES = {xvi+407},
      ISBN = {0-387-94839-2},
   MRCLASS = {91B28 (60G40 60G99 91-02)},
  MRNUMBER = {1640352},
MRREVIEWER = {Marek Rutkowski},
       DOI = {10.1007/b98840},
       URL = {https://doi.org/10.1007/b98840},
}

@article {SchmidtHieber2020,
    AUTHOR = {Schmidt-Hieber, Johannes},
     TITLE = {Nonparametric regression using deep neural networks with
              {R}e{LU} activation function},
   JOURNAL = {Ann. Statist.},
  FJOURNAL = {The Annals of Statistics},
    VOLUME = {48},
      YEAR = {2020},
    NUMBER = {4},
     PAGES = {1875--1897},
      ISSN = {0090-5364},
   MRCLASS = {62G08 (62M45)},
  MRNUMBER = {4134774},
       DOI = {10.1214/19-AOS1875},
       URL = {https://doi.org/10.1214/19-AOS1875},
}

@article {Denis2021,
    AUTHOR = {Denis, Christophe and Dion-Blanc, Charlotte and Martinez,
              Miguel},
     TITLE = {A ridge estimator of the drift from discrete repeated
              observations of the solution of a stochastic differential
              equation},
   JOURNAL = {Bernoulli},
  FJOURNAL = {Bernoulli. Official Journal of the Bernoulli Society for
              Mathematical Statistics and Probability},
    VOLUME = {27},
      YEAR = {2021},
    NUMBER = {4},
     PAGES = {2675--2713},
      ISSN = {1350-7265},
   MRCLASS = {62G05 (60H10)},
  MRNUMBER = {4303900},
       DOI = {10.3150/21-BEJ1327},
       URL = {https://doi.org/10.3150/21-BEJ1327},
}

@article {Oga2024,
    AUTHOR = {Oga, Akihiro and Koike, Yuta},
     TITLE = {Drift estimation for a multi-dimensional diffusion process
              using deep neural networks},
   JOURNAL = {Stochastic Process. Appl.},
  FJOURNAL = {Stochastic Processes and their Applications},
    VOLUME = {170},
      YEAR = {2024},
     PAGES = {Paper No. 104240, 25},
      ISSN = {0304-4149},
   MRCLASS = {60H10 (62F12 62M99 68T07)},
  MRNUMBER = {4682669},
MRREVIEWER = {Vivek S. Borkar},
       DOI = {10.1016/j.spa.2023.104240},
       URL = {https://doi.org/10.1016/j.spa.2023.104240},
}

@article {Pinelis2020,
    AUTHOR = {Pinelis, Iosif},
     TITLE = {Exact lower and upper bounds on the incomplete gamma function},
   JOURNAL = {Math. Inequal. Appl.},
  FJOURNAL = {Mathematical Inequalities \& Applications},
    VOLUME = {23},
      YEAR = {2020},
    NUMBER = {4},
     PAGES = {1261--1278},
      ISSN = {1331-4343},
   MRCLASS = {33B20 (26D07 26D15)},
  MRNUMBER = {4168199},
MRREVIEWER = {Kwara Nantomah},
       DOI = {10.7153/mia-2020-23-95},
       URL = {https://doi.org/10.7153/mia-2020-23-95},
}

@article {Hoffmann1999,
    AUTHOR = {Hoffmann, Marc},
     TITLE = {Adaptive estimation in diffusion processes},
   JOURNAL = {Stochastic Process. Appl.},
  FJOURNAL = {Stochastic Processes and their Applications},
    VOLUME = {79},
      YEAR = {1999},
    NUMBER = {1},
     PAGES = {135--163},
      ISSN = {0304-4149},
   MRCLASS = {62G07 (62M05)},
  MRNUMBER = {1670522},
MRREVIEWER = {Friedrich Liese},
       DOI = {10.1016/S0304-4149(98)00074-X},
       URL = {https://doi.org/10.1016/S0304-4149(98)00074-X},
}

@article {Comte2007,
    AUTHOR = {Comte, Fabienne and Genon-Catalot, Valentine and Rozenholc,
              Yves},
     TITLE = {Penalized nonparametric mean square estimation of the
              coefficients of diffusion processes},
   JOURNAL = {Bernoulli},
  FJOURNAL = {Bernoulli. Official Journal of the Bernoulli Society for
              Mathematical Statistics and Probability},
    VOLUME = {13},
      YEAR = {2007},
    NUMBER = {2},
     PAGES = {514--543},
      ISSN = {1350-7265},
   MRCLASS = {62M05 (60J10 62G05)},
  MRNUMBER = {2331262},
MRREVIEWER = {Juan Carlos Abril},
       DOI = {10.3150/07-BEJ5173},
       URL = {https://doi.org/10.3150/07-BEJ5173},
}



\clearpage
\onecolumn

\appendix


\section{Detailed Proofs}

\setcounter{thm}{0}
\renewcommand{\thethm}{A.\arabic{thm}}

\renewcommand{\theequation}{A.\arabic{equation}}
\setcounter{equation}{0}

For every $n\in\{1, \dots, N\}$ and every $t\in[0,T]$, define \[\calA_{t}^{(n)}\coloneqq \big\{\bar{X}_t^{(n)}\in [0,1]^d\big\}.\] Moreover, for every $n\in\{1, \dots, N\}$ and  $m\in\{0, \dots, M-1\}$, define
\[\bar{\Sigma}^{(n)}_{t_m}=\frac{1}{\Delta}\int_{t_m}^{t_{m+1}}\sigma^i(\bar{X}^{(n)}_{t_m})\dd B_{s}^{(n)} \quad \text{ and }\quad \tilde{\Sigma}^{(n)}_{t_m}=\frac{1}{\Delta}\int_{t_m}^{t_{m+1}}\big(\sigma^i(\bar{X}^{(n)}_s)-\sigma^i(\bar{X}^{(n)}_{t_m})\big)\dd B_{s}^{(n)}. \]
It is obvious that ${\Sigma}^{(n)}_{t_m}=\bar{\Sigma}^{(n)}_{t_m}+\tilde{\Sigma}^{(n)}_{t_m}$. Recall that 
$I^{(n)}_{t_m}=\frac{1}{\Delta}\int_{t_m}^{t_{m+1}}(b^i(\bar{X}^{(n)}_s)-b^i(\bar{X}^{(n)}_{t_m}))\dd s$ as given in equation~\eqref{eq:decompos-Y}, and the constant  $C_{\star}=4 (C_b^2 + C_{\sigma}^2)\exp(4(L_b^2+L_{\sigma}^2))$ defined in Lemma \ref{lem:i-and-sigma-bound}. We  also abbreviate $\mathcal{F}$ for the neural network function class $\nnset$, and $\mathfrak{N}_\delta$ for the covering number $\mathfrak{N}(\delta, \nnset, \|\cdot\|_{\infty})$. The following lemma is a direct consequence of \citet[Lemma~4.6, Corollary~4.1, and Corollary~4.2]{Oga2024}. 
\begin{lem}\label{lem:upperboundClip} Suppose that Assumptions \ref{assum:lip} and \ref{assum:parameter} hold.
For every $n\in\{1,\dots, N\}$,  $m\in\{0, \dots, M-1\}$, and $s\in[t_m, t_{m+1}]$, we have 
\begin{enumerate}
    \item[1.] $\EE \big[ |\bar X^{(n)}_s-\bar X^{(n)}_{t_m}|^2 \mathbf{1}_{\mathcal{A}_{t_m}^{(n)}}\big]\leq C_{\star}(s-t_m)$,
    \item[2.] $\EE \big[ (I_{t_{m}}^{(n)})^2 \Iamn\big]\leq \frac{1}{2}L_b^2 C_{\star}\Delta$,
    \item[3.] $\EE \big[ (\tilde \Sigma_{t_{m}}^{(n)})^2 \Iamn\big]\leq \frac{1}{2}L_{\sigma}^2 C_{\star}$.
\end{enumerate}
\end{lem}
We now present the proof of Lemma~\ref{lem:i-and-sigma-bound}.

\begin{proof}[Proof of Lemma \ref{lem:i-and-sigma-bound}]
1. Consider two arbitrary functions $f_1,\,f_2\in\nnset$. Then, since $f_1,\,f_2$ have support in $[0,1]^d$, we have, for every $c>0$, 
\begin{align}
&\mathbb{E}\left[\frac{1}{NM}\sum_{n=1}^{N}\sum_{m=0}^{M-1}I_{t_m}^{(n)}\Big(f_1(\bar{X}_{t_m}^{(n)})-f_2(\bar{X}_{t_m}^{(n)})\Big)\right]=\mathbb{E}\left[\frac{1}{NM}\sum_{n=1}^{N}\sum_{m=0}^{M-1}I_{t_m}^{(n)}\mathbf{1}_{\calA_{t_m}^{(n)}}\Big(f_1(\bar{X}_{t_m}^{(n)})-f_2(\bar{X}_{t_m}^{(n)})\Big)\right]\nonumber\\
&\quad\leq \frac{1}{NM}\sum_{n=1}^{N}\sum_{m=0}^{M-1} \left\{\mathbb{E}\left[\frac{1}{2c}\big(I_{t_m}^{(n)}\big)^2\mathbf{1}_{\calA_{t_m}^{(n)}}\right]+\EE \left[\frac{c}{2}\Big(f_1(\bar{X}_{t_m}^{(n)})-f_2(\bar{X}_{t_m}^{(n)})\Big)^2\right]\right\},\label{eq:i-term-decomp}
\end{align}
where the last inequality comes from the AM–GM inequality  $ab\leq \frac{a^2}{2c}+\frac{cb^2}{2}$ for every $a,b\in \RR_{+}$. For the first term of \eqref{eq:i-term-decomp}, it follows that 
\begin{align}
    &\frac{1}{NM}\sum_{n=1}^{N}\sum_{m=0}^{M-1} \mathbb{E}\left[\frac{1}{2c}\big(I_{t_m}^{(n)}\big)^2\mathbf{1}_{\calA_{t_m}^{(n)}}\right]= \frac{1}{NM}\frac{1}{2c}\sum_{n=1}^{N}\sum_{m=0}^{M-1} \mathbb{E}\left[\frac{1}{\Delta^2}\left(\int_{t_m}^{t_{m+1}}\big[b^i(\bar{X}^{(n)}_s)-b^i(\bar{X}^{(n)}_{t_m})\big]\dd s\right)^2\mathbf{1}_{\calA_{t_m}^{(n)}}\right]\nonumber\\
    &\quad \leq \frac{1}{NM}\frac{1}{2c}\sum_{n=1}^{N}\sum_{m=0}^{M-1} \mathbb{E}\left[\frac{L_b^2}{\Delta^2}\left(\int_{t_m}^{t_{m+1}} \left|\bar{X}^{(n)}_s-\bar{X}^{(n)}_{t_m}\right|\mathbf{1}_{\calA_{t_m}^{(n)}}\,\dd s\right)^2\right]\nonumber\\
    &\quad\leq \frac{1}{NM}\frac{1}{2c}\frac{L_b^2}{\Delta}\sum_{n=1}^{N}\sum_{m=0}^{M-1} \mathbb{E}\left[\int_{t_m}^{t_{m+1}} \left|\bar{X}^{(n)}_s-\bar{X}^{(n)}_{t_m}\right|^2\mathbf{1}_{\calA_{t_m}^{(n)}}\,\dd s\right]\nonumber\\
    &\quad \leq \frac{1}{NM}\frac{1}{2c}\frac{L_b^2}{\Delta}\sum_{n=1}^{N}\sum_{m=0}^{M-1} \int_{t_m}^{t_{m+1}} \mathbb{E}\left[\left|\bar{X}^{(n)}_s-\bar{X}^{(n)}_{t_m}\right|^2\mathbf{1}_{\calA_{t_m}^{(n)}}\right]\,\dd s\nonumber\\
    &\quad \leq \frac{1}{NM}\frac{1}{2c}\frac{L_b^2}{\Delta}\sum_{n=1}^{N}\sum_{m=0}^{M-1} \int_{t_m}^{t_{m+1}} C_{\star}(s-t_m)\,\dd s=  \frac{1}{2c}\frac{L_b^2C_{\star}}{2}\Delta, \nonumber
\end{align}
where the first inequality follows from Assumption~\ref{assum:lip}, the second uses the Cauchy–Schwarz inequality, 
the third is a direct application of Fubini’s theorem, and the fourth follows from Lemma~\ref{lem:upperboundClip}.

For the second term of \eqref{eq:i-term-decomp}, we have 
\begin{align}
    \frac{1}{NM}\sum_{n=1}^{N}\sum_{m=0}^{M-1}\EE \left[\frac{c}{2}\Big(f_1(\bar{X}_{t_m}^{(n)})-f_2(\bar{X}_{t_m}^{(n)})\Big)^2\right]=\frac{c}{2}    \EE \left[\frac{1}{NM}\sum_{n=1}^{N}\sum_{m=0}^{M-1}\Big(f_1(\bar{X}_{t_m}^{(n)})-f_2(\bar{X}_{t_m}^{(n)})\Big)^2\right]=\frac{c}{2}  \hat{\calR}(f_1, f_2).\nonumber
\end{align}
Consequently, 
\begin{align}
    \mathbb{E}\left[\frac{1}{NM}\sum_{n=1}^{N}\sum_{m=0}^{M-1}I_{t_m}^{(n)}\Big(f_1(\bar{X}_{t_m}^{(n)})-f_2(\bar{X}_{t_m}^{(n)})\Big)\right]\leq \frac{c}{2}  \hat{\calR}(f_1, f_2)+\frac{1}{2c}\frac{L_b^2C_{\star}}{2}\Delta.
\end{align}
Finally, by setting $(f_1, f_2, c)$ respectively to $(\hat f, f_0, \frac{\varepsilon}{2})$ $(\bar f, f_0, \varepsilon)$, and applying the triangle inequality, we obtain \eqref{eq:ineq-I-term}.

\smallskip

\noindent 2. It follows from Lemma 4.11 and 4.12 of \citet{Oga2024} that for every fixed $n\in\{1, ..., N\}$, 
\begin{align}
&\EE \left[\left|\frac{1}{M}\sum_{m=0}^{M-1}\tilde\Sigma_{t_m}^{(n)}\big(\hat{f}-f_{\mathfrak{n}^*}\big)(\bar{X}_{t_m}^{(n)})\right|\right]\leq \delta \sqrt{\frac{1}{2}L_{\sigma}^2 C_{\star}} \: \text{and}\: \EE \left[\left|\frac{1}{M}\sum_{m=0}^{M-1}\bar\Sigma_{t_m}^{(n)}\big(\hat{f}-f_{\mathfrak{n}^*}\big)(\bar{X}_{t_m}^{(n)})\right|\right] \leq\frac{\bar{C}_{\sigma}}{\sqrt{T}}\int_{0}^{\delta}\sqrt{\log \mathfrak{N}_{u}}\,\dd u,\nonumber
\end{align}
which implies directly \eqref{eq:ineq-sigma-term} since $\Sigma_{t_m}^{(n)}=\bar \Sigma_{t_m}^{(n)}+\tilde \Sigma_{t_m}^{(n)}$.
\end{proof}

\smallskip 
Next, we provide the detailed proofs of Lemma~\ref{lem:main-diff-from-oga} and Proposition \ref{prop:upper-bound-train-error}.
 
\begin{proof}[Detailed proof of Lemma \ref{lem:main-diff-from-oga}]
For a fixed function $f\in\calF$ and for a fixed $n\in\{1,..., N\}$, we define the processes $\widehat{M}^{(n)}(f)=(\widehat{M}^{(n)}(f)_s)_{s\in[0,T]}$,  $\widehat{A}^{(n)}(f)=(\widehat{A}^{(n)}(f)_s)_{s\in[0,T]}$, $\overline{M}(f)=(\overline{M}(f)_s)_{s\in[0,T]}$ and $ \bar{A}(f)=(\bar{A}(f)_s)_{s\in[0,T]}$  by 
\begin{align}
&\widehat{M}^{(n)}(f)_s\coloneqq \sum_{m=0}^{M-1}(f-f_0)(\bar{X}_{t_m}^{(n)})\int_{s\wedge t_m}^{s\wedge t_{m+1}}\sigma ^i\big(\bar X_u^{(n)}\big)\dd B_u^{(n)},\nonumber\\
&\overline{M}(f)_s\coloneqq \frac{1}{N}\sum_{n=1}^{N} \widehat{M}^{(n)}(f) =\frac{1}{N}\sum_{n=1}^{N}\sum_{m=0}^{M-1}(f-f_0)(\bar{X}_{t_m}^{(n)})\int_{s\wedge t_m}^{s\wedge t_{m+1}}\sigma ^i\big(\bar X_u^{(n)}\big)\dd B_u^{(n)},\nonumber\\
&\widehat{A}^{(n)}(f)_s\coloneqq \langle\widehat{M}^{(n)}(f)\rangle_s=\sum_{m=0}^{M-1}\Big((f-f_0)(\bar{X}_{t_m}^{(n)})\Big)^2\int_{s\wedge t_m}^{s\wedge t_{m+1}}\left|\sigma^i\big(\bar X_u^{(n)}\big)\right|^2\dd u,\nonumber\nonumber\\
&  \bar{A}(f)_s\coloneqq \big\langle \overline{M}(f)\big\rangle_s=\frac{1}{N^2}\sum_{n=1}^{N}\widehat{A}^{(n)}(f)_s. \nonumber
\end{align}
It follows that 
\begin{align}
 &\mathbb{E}\left[\frac{1}{NM}\sum_{n=1}^{N}\sum_{m=0}^{M-1}\Sigma_{t_m}^{(n)}\big(f_{\mathfrak{n}^*}(\bar{X}_{t_m}^{(n)})-{f}_0(\bar{X}_{t_m}^{(n)})\big)\right]\leq \mathbb{E}\left[\left|\frac{1}{T} \overline{M}(f_{\mathfrak{n}^*})_{T}\right|\right]=\frac{2}{T}\EE \left[\left|\overline{\xi}_{f_{\mathfrak{n}^*}}\right| \left(\bar{A}(f_{\mathfrak{n}^*})_T+\overline{D}_{f_{\mathfrak{n}^*}} \right)^{\frac{1}{2}}\right]\nonumber\\
 & \leq \frac{2}{T} \sqrt{\EE \left[\overline{\xi}_{f_{\mathfrak{n}^*}}^2\right] \left(2\EE \left[\bar{A}(f_{\mathfrak{n}^*})_T\right]+\varepsilon'\right)},\label{eq:decomposition-sigma-n}
\end{align}
where for every $f\in\calF$,  we denote by $\overline{\xi}_{f}\coloneqq \frac{\overline{M}(f)_T}{2\sqrt{\bar{A}(f)_T+\overline{D}_f}}$
and $\overline{D}_f=\EE\left[ \bar{A}(f)_T\right]+\varepsilon'$ for some $\varepsilon'>0$.

For $\EE \big[\overline{\xi}_{f}^2\big] $, as $\overline{M}(f)$ is a continuous local martingale for every $f \in\calF$, it follows from \cite[Lemma 4.9]{Oga2024} that
\begin{align}
    \EE \left[\frac{\sqrt{\overline{D}_{f}}}{\sqrt{\bar{A}(f_{n})_T+\overline{D}_{f}}}\exp\left\{ 2\overline{\xi}_{f}^2\right\}\right]\leq 1.
\end{align}
Moreover, 
\begin{align}
\EE \left[ \exp\left(\overline{\xi}_{f_{n^{*}}}^2\right)\right]\leq \sqrt{
\EE\left[ \frac{\sqrt{\overline{D}_{f_{\mathfrak{n}^*}}}}{\sqrt{\bar{A}(f_{n^{*}})_T+\overline{D}_{f_{\mathfrak{n}^*}}}}\exp\left\{ 2\overline{\xi}_{f_{n^{*}}}^2\right\}\right]\EE \left[\frac{\sqrt{\bar{A}(f_{n^{*}})_T+\overline{D}_{f_{\mathfrak{n}^*}}}}{\sqrt{\overline{D}_{f_{\mathfrak{n}^*}}}}\right]
}.\nonumber
\end{align}
Since 
\begin{align}
\EE\left[ \frac{\sqrt{\overline{D}_{f_{\mathfrak{n}^*}}}}{\sqrt{\bar{A}(f_{n^{*}})_T+\overline{D}_{f_{\mathfrak{n}^*}}}}\exp\left\{ 2\overline{\xi}_{f_{n^{*}}}^2\right\}\right]&\leq \EE \left[ \max_{1\leq n\leq \mathfrak{N}(\delta, \mathcal{F}, \Vert \cdot \Vert_{\infty})}\frac{\sqrt{\overline{D}_{f_{n}}}}{\sqrt{\bar{A}(f_{n})_T+\overline{D}_{f_{n}}}}\exp\left\{ 2\overline{\xi}_{f_{n}}^2\right\}\right]\leq \mathfrak{N}(\delta, \mathcal{F}, \Vert \cdot \Vert_{\infty}),\nonumber
\end{align}
and 
\begin{align}
    \EE \left[\frac{\sqrt{\bar{A}(f_{n^{*}})_T+\overline{D}_{f_{\mathfrak{n}^*}}}}{\sqrt{\overline{D}_{f_{\mathfrak{n}^*}}}}\right]\leq \sqrt\EE \left[\frac{\bar{A}(f_{n^{*}})_T+\overline{D}_{f_{\mathfrak{n}^*}}}{\overline{D}_{f_{\mathfrak{n}^*}}}\right]\leq \sqrt{2},
\end{align}
we obtain
\begin{align}
\EE \left[ \exp\left(\overline{\xi}_{f_{n^{*}}}^2\right)\right]\leq 2^{\frac{1}{4}}\sqrt{\mathfrak{N}(\delta, \mathcal{F}, \Vert \cdot \Vert_{\infty})}.
\end{align}
Hence, the Jensen inequality yields 
\begin{align}
\EE \left[ \overline{\xi}_{f_{n^{*}}}^2\right]\leq \log \EE \left[ \exp\left(\overline{\xi}_{f_{n^{*}}}^2\right)\right]\leq \frac{1}{2}\log 2 +\frac{1}{2}\mathfrak{N}(\delta, \mathcal{F}, \Vert \cdot \Vert_{\infty}).
\end{align}

For $\EE \left[\bar{A}(f_{\mathfrak{n}^*})_T\right]$, we have
\begin{align}
    \EE \left[\bar{A}(f_{\mathfrak{n}^*})_T\right]=\EE \left[\frac{1}{N^2}\sum_{n=1}^{N}\widehat{A}^{(n)}(f_{\mathfrak{n}^*})_T\right]=\frac{1}{N}\cdot \EE \left[\frac{1}{N}\sum_{n=1}^{N}\widehat{A}^{(n)}(f_{\mathfrak{n}^*})_T\right]
\end{align}
and 
\begin{align}
&\EE \left[\frac{1}{N}\sum_{n=1}^{N}\widehat{A}^{(n)}(f_{\mathfrak{n}^*})\right]\nonumber\\
&=\EE \left[\frac{1}{N}\sum_{n=1}^{N}\sum_{m=0}^{M-1}\Big((f_{\mathfrak{n}^*}-f_0)(\bar{X}_{t_m}^{(n)})\Big)^2\int_{t_m}^{t_{m+1}}\left|\sigma^i\big(\bar X_u^{(n)}\big)\right|^2\dd u \right]\nonumber\\
&\leq \EE \left[\frac{1}{N}\sum_{n=1}^{N}\sum_{m=0}^{M-1}\Big((f_{\mathfrak{n}^*}-f_0)(\bar{X}_{t_m}^{(n)})\Big)^2\int_{t_m}^{t_{m+1}}\left[2 \left|\sigma^i\big(\bar X_{t_m}^{(n)}\big)\right|^2+2\left|\sigma^i\big(\bar X_u^{(n)}\big)-\sigma^i\big(\bar X_{t_m}^{(n)}\big)\right|^2\right]\dd u \right]\nonumber\\
&\leq \EE \left[\frac{1}{N}\sum_{n=1}^{N}\sum_{m=0}^{M-1}\Big((f_{\mathfrak{n}^*}-f_0)(\bar{X}_{t_m}^{(n)})\Big)^2\int_{t_m}^{t_{m+1}}\left[2 \left|\sigma^i\big(\bar X_{t_m}^{(n)}\big)\right|^2\right]\dd u \right]\nonumber\\
&\quad + \EE \left[\frac{1}{N}\sum_{n=1}^{N}\sum_{m=0}^{M-1}\Big((f_{\mathfrak{n}^*}-f_0)(\bar{X}_{t_m}^{(n)})\Big)^2\int_{t_m}^{t_{m+1}}\left[2\left|\sigma^i\big(\bar X_u^{(n)}\big)-\sigma^i\big(\bar X_{t_m}^{(n)}\big)\right|^2\right]\dd u \right]\nonumber\\
&\eqqcolon \mathrm{(I)} + \mathrm{(II)}.\nonumber
\end{align}
Since $f_{\mathfrak{n}^*}$ and $f_0$ have support in $[0,1]^d$,  we have for (I)
\begin{align}
 \mathrm{(I)}\leq 2 C_{\sigma}^2T \EE \left[\frac{1}{NM}\sum_{n=1}^{N}\sum_{m=0}^{M-1}\Big((f_{\mathfrak{n}^*}-f_0)(\bar{X}_{t_m}^{(n)})\Big)^2\right]=2 C_{\sigma}^2T\hat{\mathcal{R}}(f_{\mathfrak{n}^*},f_{0}),\nonumber
\end{align}
and 
\begin{align}
    &\hat{\mathcal{R}}(f_{\mathfrak{n}^*},f_{0})-\hat{\mathcal{R}}(\hat{f},f_{0})=\mathbb{E}\left[\frac{1}{NM}\sum_{n=1}^{N} \sum_{m=0}^{M-1}\left[(\hat{f}_{\mathfrak{n}^*}(\bar{X}_{t_m}^{(n)})-f_{0}(\bar{X}_{t_m}^{(n)}))^2-\hat{f}(\bar{X}_{t_m}^{(n)})-f_{0}(\bar{X}_{t_m}^{(n)}))^2\right]\right]\nonumber\\
    &\leq \mathbb{E}\left[\frac{1}{NM}\sum_{n=1}^{N} \sum_{m=0}^{M-1}\left[\left(\hat{f}_{\mathfrak{n}^*}(\bar{X}_{t_m}^{(n)})-\hat{f}(\bar{X}_{t_m}^{(n)})\right)\left(\hat{f}_{\mathfrak{n}^*}(\bar{X}_{t_m}^{(n)})+\hat{f}(\bar{X}_{t_m}^{(n)})-2f_{0}(\bar{X}_{t_m}^{(n)})
    \right)\right]\right]\leq \delta(2F+2C_{b})\leq 4F\delta.\nonumber
\end{align}

For (II), we have
\begin{align}
 \mathrm{(II)}&\leq 8F^2\EE \left[\sum_{m=0}^{M-1}\int_{t_m}^{t_{m+1}}\left|\sigma^i\big(\bar X_u^{(1)}\big)-\sigma^i\big(\bar X_{t_m}^{(1)}\big)\right|^2\mathbf{1}_{\mathcal{A}_{t_m}^{(1)}}\dd u \right]\leq  8F^2\EE \left[\sum_{m=0}^{M-1}\int_{t_m}^{t_{m+1}}L_{\sigma}^2\left|\bar X_u^{(1)}-\bar X_{t_m}^{(1)}\right|^2\mathbf{1}_{\mathcal{A}_{t_m}^{(1)}}\dd u \right]\nonumber\\
 &\leq 8F^2L_{\sigma}^2\sum_{m=0}^{M-1}\int_{t_m}^{t_{m+1}}\EE \left[\left|\bar X_u^{(1)}-\bar X_{t_m}^{(1)}\right|^2\mathbf{1}_{\mathcal{A}_{t_m}^{(1)}}\right]\dd u \leq 8F^2L_{\sigma}^2 C_{\star}T\Delta, \nonumber
\end{align}
where the last inequality follows from Lemma \ref{lem:upperboundClip}. Consequently, 
\begin{align}
    &\EE \left[\frac{1}{N}\sum_{n=1}^{N}\widehat{A}^{(n)}(f_{\mathfrak{n}^*})_T\right]\leq 2 C_{\sigma}^2T\big(\hat{\mathcal{R}}(\hat{f},f_{0})+4F\delta\big) + 8F^2L_{\sigma}^2 C_{\star}T\Delta.\nonumber
\end{align} 
Finally, inserting the above bounds into \eqref{eq:decomposition-sigma-n} yields
\begin{align}
 &\mathbb{E}\left[\frac{1}{NM}\sum_{n=1}^{N}\sum_{m=0}^{M-1}\Sigma_{t_m}^{(n)}({f_{\mathfrak{n}^*}}(\bar{X}_{t_m}^{(n)})-{f}_0(\bar{X}_{t_m}^{(n)}))\right] \nonumber\\
 &\quad \leq \frac{1}{T} \sqrt{\log 2 +2\log\mathfrak{N}_{\delta}}\cdot \sqrt{\frac{4}{N}\left[   C_{\sigma}^2T \big(\hat{\mathcal{R}}(\hat{f},f_{0})+4F\delta\big) + 4F^2L_{\sigma}^2 C_{\star}T\Delta \right]+\varepsilon'}.\nonumber
\end{align}
By letting $\varepsilon'\rightarrow0$ and by applying the AM-GM inequality $\sqrt{xy}\leq \frac{\varepsilon}{4}x+\frac{1}{\varepsilon}y$ for some $\varepsilon\in(0,1)$, we get the desired inequality. 
\end{proof}

\begin{proof}[Detailed proof of Proposition \ref{prop:upper-bound-train-error}]
By inserting the results of Lemmas \ref{lem:i-and-sigma-bound} and \ref{lem:main-diff-from-oga} into \eqref{eq:decomposition-train-error}, 
we obtain, for a  an arbitrary function $\bar{f}$  in $\calF$, for $\varepsilon\in(0,1)$,
\begin{align}
\hat{\calR}(\hat{f}, f_0)\leq \frac{\Psi^{\calF}(\hat{f})}{1-\varepsilon}+ \frac{1+\varepsilon}{1-\varepsilon}{\calR}(\bar f, f_0)+\frac{\gamma'_{\varepsilon}}{1-\varepsilon},\nonumber
\end{align}
where we apply  $\hat{\calR}(\bar f, f_0)={\calR}(\bar f, f_0)$ and 
\begin{align}
    \gamma'_{\varepsilon}\!=&\frac{3L_b^2C_{\star}}{2\varepsilon}\Delta+\delta\sqrt{2L_{\sigma}^2C_{\star}}+\frac{2\bar{C}_{\sigma}}{\sqrt{T}}\int_{0}^{\delta}\!\!\!\sqrt{\log \mathfrak{N}_{u}}\,\dd u+2\varepsilon F\delta + 2\varepsilon \frac{F^2L_{\sigma}^2 C_{\star}}{C_{\sigma}^2}\Delta+\frac{8C_{\sigma}^2}{\varepsilon TN}\big(\log 2 +2\log\mathfrak{N}_{\delta}\big).\nonumber
\end{align}
By applying Lemma \ref{lem:lemma413-in-oga}, we have
\begin{align}
\int_{0}^{\delta}\sqrt{\log \mathfrak{N}_{u}}\,\dd u\leq &\int_{0}^{\delta} \sqrt{C s(L \log s+\log d)}\,\dd u  +\sqrt{Cs}\int_{0}^{\delta}\sqrt{-\log u }\,\dd u. \nonumber
\end{align}
After a change of variable $t=-\log u$, we obtain 
\begin{align}
    \int_{0}^{\delta}\sqrt{-\log u }\,\dd u
    =\Gamma\big(\tfrac{3}{2}, -\log (\delta)\big),\nonumber
\end{align}
where $\Gamma(s, x)$ denotes the upper incomplete gamma function. Therefore, for $\delta < \exp(-\tfrac{1}{2})$, it follows from \citet[Proposition 2.7]{Pinelis2020} that $\Gamma\big(\tfrac{3}{2}, -\log (\delta)\big)\leq \tfrac{3}{4}\delta \sqrt{-\log \delta}$. 

Finally, by considering $\varepsilon=\frac{1}{2}$ and $\delta=\frac{1}{N}$, we obtain
\begin{align}
&\hat{\calR}(\hat{f}, f_0)\leq 2\Psi^{\calF}(\hat{f})+ 3{\calR}(\bar f, f_0) +2L_{\sigma}^2C_{\star}\big(3+\tfrac{F^2}{C_{\sigma^2}}\big)\Delta+ C_1\frac{1}{N}+C_2\frac{\log N}{N},\nonumber 
\end{align}
where 
\begin{align}
    C_1= & 2\sqrt{2L_\sigma^2 C_{\star}}+\frac{4\bar{C}_{\sigma}}{\sqrt{T}}\sqrt{Cs(L \log s+\log d)}+2F +\frac{32 C_{\sigma}^2}{T}\big( \log 2 +2 Cs (L \log s+\log d)\big)\nonumber
\end{align}
and $  C_2= \frac{6\bar{C}_{\sigma}\sqrt{Cs}}{\sqrt{T}} +\frac{64 C_{\sigma}^2Cs}{T}$. We conclude by noting that ${\calR}(\bar f, f_0) \leq \Vert \bar f - f_0 \Vert_\infty$, and observing that the above inequality holds for any arbitrary function $\bar{f} \in \calF$.
\end{proof}


\section{Experimental and Computational Details}
\setcounter{thm}{0}
\renewcommand{\thethm}{B.\arabic{thm}}

\subsection{Data Generation}

We simulate sample paths of \eqref{eq:sde} using the Euler–Maruyama discretization scheme. The time horizon is set to $T = 1$ and is discretized into $M = 100$ time steps, with the time step $\Delta = T/M = 0.01$. 

For each experiment, we generate $N$ independent paths $\mathcal{D}_N := \{ \bar{X}^{(n)} \}_{n=1}^{N}$ as training set and an $N' = 1000$ independent paths $\mathcal{D}_{\text{test}} := \{ \tilde{X}^{(n)}\}_{n=1}^{N'}$ for evaluation. Hyperparameter tuning is conducted using a separate independent validation set $\mathcal{D}_{\text{valid}} := \{ \tilde{\tilde{X}}^{(n)}\}_{n=1}^{1000}$.

\subsection{Hyperparameter Tuning}
As the empirical performance may vary across neural network architectures and training settings, we perform hyperparameter tuning for each combination of $(N, d)$ on  $\mathcal{D}_{\text{valid}}$ to ensure comparability of estimation errors across different combinations of $(N, d)$. Specifically, the tuning is performed over the following candidate sets:
\begin{itemize}
    \item network architectures: $\mathbf{p}\in \{(d, 16, 16, 1), (d, 16, 32, 16, 1), (d, 16, 32, 32, 16, 1)\},$
    \item percentage of non-zero parameters in total $s_\text{ratio}$: $\{0.25, 0.5, 0.75\}$.
\end{itemize}
The sparsity constraint $s$ used in the estimator class $\mathcal{F}(L, \mathbf{p}, s, F)$ is computed as
$$
s = \left\lfloor s_{\text{per}} \cdot \sum_{i=0}^{L} (p_i + 1)\cdot p_{i+1} \right\rfloor,
$$
where $L$ is the number of hidden layers and $(p_0,\dots, p_{L+1})$ is the layer width vector.

\subsection{Training Procedure}
The network is trained using the Adam optimizer with learning rate of 0.001 and batch size of 256. At each epoch, the training samples are shuffled and fed to the model in mini-batches to update the network parameters. 
Following each optimizer step, we enforce the desired sparsity level by retaining only the top-$s$ parameters (in magnitude) and setting all remaining weights to zero. We then clip all parameters to lie within the range $[-1, 1]$.

After every epoch, the model is evaluated on the validation set to compute the validation error. Training is performed for up to 200 epochs, and the final number of epochs is selected via early stopping, which is triggered when the validation error fails to improve for 20 consecutive epochs.

\subsection{Evaluation}

To account for the randomness in data generation and optimization process, for each pair $(N, d)$, we repeat the experiment 50 times with different random seeds, resulting in a series of different estimators and test set $((\widehat{f}_{N,d}^{(j)}), \mathcal{D}_{\text{test}}^{(j)})_{j=1}^{50}$. Let $E_{N,d}^{(j)}$ denote the generalization error computed in the $j$-th run with training size $N$ and dimension $d$. It is defined as 
$$E_{N,d}^{(j)}:=\widetilde{\mathcal{R}}_N(\widehat{f}_{N,d}^{(j)},b_1),$$
where the error is evaluated on the test dataset $\mathcal{D}_{\text{test}}^{(j)}$.  
For each $(N, d)$, we compute:
\begin{equation}
\label{eq:minmax-interval}
\bar{E}_{N,d} := \frac{1}{50} \sum_{j=1}^{50} E_{N,d}^{(j)}, \quad
E_{N,d}^{\text{lower}} := \bar{E}_{N,d} - t_{0.975}^{(49)} \cdot \tau_{N,d}, \quad
E_{N,d}^{\text{upper}} := \bar{E}_{N,d} + t_{0.975}^{(49)} \cdot \tau_{N,d}.
\end{equation}
where  $\tau_{N,d} = \frac{1}{\sqrt{50}} \big( \frac{1}{49} \sum_{j=1}^{50} \big(E_{N,d}^{(j)} - \bar{E}_{N,d} \big)^2 \big)^{1/2}$ and $t_{0.975}^{(49)}$ is the 97.5 percentile of the Student’s $t$ distribution with 49 degrees of freedom. 

Log–log plots are used to present the average empirical error $\bar{E}_{N,d}$ and interval $[E_{N,d}^{\text{lower}}, E_{N,d}^{\text{upper}}]$ with respect to training size $N$.

\subsection{Memory Cost Discussion}
The $B$-spline estimator involves constructing the matrix $\mathbf{B} \in \mathbb{R}^{NM \times p}$, where $p = (K_N + 3)^d$ denotes the total number of basis functions, and $K_N$ is the number of interior knots per coordinate. The memory cost is thus dominated by the storage of the matrix $\mathbf{B}$:
$$
\text{Memory}_{\text{$B$-spline}} \sim O(NM p) = O\left( NM (K_N + 3)^d \right).
$$
As $d$ increases, the number of basis functions $p$ grows exponentially, resulting in an exponential blow-up of memory cost. For instance, even with moderate values of $K_N$ (e.g., $K_N=8$), when $d=5$, we have
$$
p = (8 + 3)^5 = 161051,
$$
leading to a matrix $\mathbf{B}$ of size $(100N) \times 161051$ entries. Assuming each entry is stored as a 64-bit float (8 bytes), the total memory required is
$$
\text{Memory}_{\text{$B$-spline}} \approx 100N \times 161051 \times 8 \text{ bytes}.
$$
In particular, for $d=5$, the required memory far exceeds the computational capacity of a standard laptop (Apple M2 chip, 16 GB RAM). 

In contrast, the neural network estimator requires storing only the network parameters and minibatch data during training. The memory cost scales as
$$
\text{Memory}_{\text{NN}} \sim O\left( \sum_{i=1}^{L} (p_{i-1} + 1)p_{i} + \text{batch\_size} \times d \right),
$$
where $p_i$ is the layer width and $d$ is the input dimension.

In particular, for $d=5$, under the architecture $(d, 16, 32, 16, 1)$, with batch size $256$, the total number of parameters is $1185$, yielding under 64-bit float (8 bytes):
$$
\text{Memory}_{\text{NN}} \approx (1185+256\times 5)\times8 \text{ bytes}.
$$
Thus, the overall memory required is approximately 19 KB. Crucially, the memory requirement of the neural network estimator grows linearly with $d$ and remains independent of the sample size $N$, making it more scalable for high-dimensional settings compared to the $B$-spline estimator.

\end{document}